\documentclass{article}

\usepackage{macros}

\usepackage[preprint]{neurips_2022}

\usepackage[utf8]{inputenc} 
\usepackage[T1]{fontenc}    
\usepackage{hyperref}       
\usepackage{url}            
\usepackage{booktabs}       
\usepackage{amsfonts}       
\usepackage{nicefrac}       
\usepackage{microtype}      
\usepackage{xcolor}         

\usepackage{amsmath}
\usepackage{amssymb}
\usepackage{amsthm}
\usepackage{multirow}
\usepackage{comment}
\usepackage{graphicx}
\usepackage{caption}
\usepackage[plain,noend]{algorithm2e}
\usepackage[super]{nth}

\newtheorem{theorem}{Theorem}
\newtheorem{lemma}{Lemma}

\newtheorem{proposition}{Proposition}

\newtheorem{definition}{Definition}

\newtheorem{remark}{Remark}

\title{Robust Regularized Low-Rank Matrix Models for Regression and Classification}

\author{%
    Hsin-Hsiung Huang\\
    Department of Statistics and Data Science\\
    University of Central Flroida\\
    Orlando, FL 32816 \\
    \texttt{Hsin-Hsiung.Huang@ucf.edu} \\
    \And
    Feng Yu \\
    Department of Mathematics and Statistics \\
    State University of New York at Albany \\
    Albany, NY 12203 \\
    \texttt{fyu5@albany.edu} \\
    \AND
    Xing Fan \\
    Department of Mathematics\\
    University of Central Florida\\
    Orlando, FL 32816 \\
    \texttt{fanxing@knights.ucf.edu} \\
    \AND
    Teng Zhang\thanks{Corresponding author} \\
    Department of Mathematics\\
    University of Central Florida\\
    Orlando, FL 32816 \\
    \texttt{teng.zhang@ucf.edu} \\
}

\begin{document}

\maketitle

\begin{abstract}
	While matrix variate regression models have been studied in many existing works, classical statistical and computational methods for the analysis of the regression coefficient estimation are highly affected by high dimensional and noisy matrix-valued predictors. To address these issues, this paper proposes a framework of matrix variate regression models based on a rank constraint, vector regularization (e.g., sparsity), and a general loss function with three special cases considered: ordinary matrix regression, robust matrix regression, and matrix logistic regression. We also propose an alternating projected gradient descent algorithm. Based on analyzing our objective functions on manifolds with bounded curvature, we show that the algorithm is guaranteed to converge, all accumulation points of the iterates have estimation errors in the order of $O(1/\sqrt{n})$ asymptotically and substantially attaining the minimax rate. Our theoretical analysis can be applied to general optimization problems on manifolds with bounded curvature and can be considered an important technical contribution to this work. We validate the proposed method through simulation studies and real image data examples.
\end{abstract}


\section{Introduction}\label{Intro}
In numerous modern scientific applications, matrix-valued predictors of interest naturally exist in massive data sets, for example, a grayscale image which quantifies intensities of image pixels is represented as a two-dimensional (2D) data matrix \citep{zhou2014regularized}, and the recommendation system of online services records the preferences of users over their products as a matrix \citep{elsener2018robust}. Other applications include the brain signal electroencephalography (EEG) data set \citep{HungWang2012,zhou2014regularized}, the microscopic leucorrhea images \citep{hao2019automatic} and the diabetes data \citep{li2021double}. To analyze such data sets, it is of great interest to adapt traditional tools such as regression and logistic regression to matrix-valued covariates. While we may vectorize a matrix-valued covariate as a vector and apply standard procedures, the size of the vector might be large and standard algorithms are inefficient and computationally expensive.  In addition, these methods cannot identify the row and column effects due to ignoring the inherent matrix structure. To overcome these problems, many algorithms have been proposed with a consideration of the matrix structure. In particular, for the matrix-variate data, the true signals frequently have sparse structures, or can be well approximated by a low rank structure. We remark that the low-rank property can be considered as sparsity in terms of the rank of the matrix parameters, which is intrinsically different from sparsity in the number of nonzero entries.

To induce such a sparse or a low-rank  structure, a variety of new regression tools have been proposed. For example, nuclear-norm penalized optimization has been applied in several works \citep{elsener2018robust,10.5555/3112656.3112880,10.1214/10-AOS850,fan2016shrinkage} to induce a low-rank structure. In particular, \citet{10.1214/10-AOS850} considered a standard M-estimator based on regularization by the nuclear or trace norm over matrices, and analyzed its performance under high-dimensional setting, and \citet{elsener2018robust} considered robust nuclear norm penalized estimators using two well-known robust loss functions: the absolute value loss and the Huber loss, and derive the asymptotic performance of these estimators.  \citet{fan2016shrinkage} applied the nuclear-norm penalized least-squares approach to appropriately truncated or shrunk data to  four popular problems: sparse linear models \citep{tibshirani1996regression}, compressed sensing \citep{donoho2006compressed}, matrix completion \citep{candes2009exact}, and multitask learning \citep{caruana1997multitask}, as well as robust covariance estimation \citep{campbell1980robust}. Besides nuclear norm based regularization, other penalization methods have been used in existing works as well. For example,  \citet{10.1214/10-AOS860} applied a Schatten-$p$ quasi-norm penalized least squares estimator, where $p\leq 1$, and obtain non-asymptotic upper bounds on the prediction risk and on the Schatten-$q$ risk of the estimators, where $p\leq q\leq 2$, and \citet{zhou2014regularized} proposed a method that solves $\Cb$ with spectral regularization, which is a class of penalties including both the nuclear norm and Schatten-$p$ quasi-norm as special cases. In addition, the multivariate group Lasso \citep{10.1214/09-AOS776}, multivariate sparse group Lasso \citep{Li2015} and fused Lasso \citep{li2021double} have also been used in existing works.

Besides regularization, another direction is to apply the low-rank constrained optimization that it has been studied for many other matrix estimation problems, including matrix sensing \citep{pmid28316458,doi:10.1137/18M1224738} and robust principal component analysis~\citep{10.5555/3291125.3309642} directly in the estimator. Compared with regularization-based methods, this direction is relatively less studied in the matrix variate regression and logistic regression settings. \citet{zhou2013tensor} applied the low-rank constraint to a tensor structure (which can be considered as a generalization of a matrix structure) and uses a block relaxation algorithm to solve the problem. However, there only exists a limited theoretical guarantee for their proposed block relaxation algorithm: consistency is only established for the unregularized estimator, and there does not exist a consistency result for the regularized estimator under the low-rank constraint. \cite{HungWang2012} proposed a maximum likelihood estimator for the model of matrix variate logistic regression based on rank-one matrix decomposition. \cite{hung2019low} applies the low-rank structure of the matrix-covariate under the same model as in this work. However, it focuses on the problem of hypothesis testing, rather than estimation as in our work.

While existing works either apply a vector-based regularization   (such as $\ell_1$ or total variation penalization \cite{li2021double}) or matrix-based  regularization such as  (such as nuclear norm penalization or low-rank constraint used in \cite{zhou2013tensor,zhou2014regularized,elsener2018robust}), the proposed estimators combine the vector-based regularization with a low-rank constraint and further improve performance. This method is consistent with the empirical observation that ``double regularization'' sometimes improves the performance, for example, the popular elastic net regularization method~\citep{zou2005regularization} uses both $\ell_1$ and $\ell_2$ regularization and achieves improved performance under certain scenarios. In addition, estimating simultaneously sparse and low-rank matrices has been studied in \cite{10.1093/imaiai/iaw012} and have applications in image processing and phase retrieval \citep{vaswani2017low}.

Our work in this paper has three main contributions. First, we propose an alternating projected gradient descent algorithm and apply it to various settings such as ordinary matrix regression, robust matrix regression, and matrix logistic regression, and the algorithm achieves the minimax estimation error when it is well-initialized. To overcome an obstacle of the nonlinear constraint in the optimization problem, we consider the constraint set in \eqref{eq:optimization} as a manifold and generalize the ``second-order Taylor expansion'' to the function defined on a manifold. Second, for the matrix variate regression problem, we investigate a method based on regularized optimization with a low-rank constraint and establish its consistency, which has not been fully studied in existing works \citep{zhou2013tensor,HungWang2012}. 
In addition, we show that when the rank is small, our estimation error is smaller than the the regularization-based methods without rank constraints \citep{li2021double}, and acquires the desired statistical accuracy in a minimax sense \citep{tsybakov2008introduction,koltchinskii2011nuclear,she2021analysis}. Third, we propose a robust regularized low-rank regression model which provides accurate predictions when data contain outliers or have heavier tails than normal distributions
\citep{she2017robust}. Our consistency theorem (Theorem~\ref{prop:convergence2}) does not depend on the convexity of the loss function, and hence it applies to various choices of $l$, including popular loss function functions in robust statistics such as redescending $\psi$’s, Hampel’s loss, or Tukey’s bisquare, etc.~\citep{10.1214/aoms/1177703732,maronna2018robust,she2017robust,huang2020robust}. In addition, our theoretical analysis can be applied to general optimization problems on manifolds with bounded curvature, and thus can be considered as an important technical contribution of this paper.

\section{Methodology}\label{sec:methodology}

In this section, we introduce the regression models used in this work. Following the existing works such as \cite{zhou2014regularized}, \cite{10.1214/10-AOS860} and \cite{li2021double}, we use a model that contains both matrix-valued predictors in $\reals^{m\times q}$ and vector-valued predictors in  $\real^p$. This matrix variate regression model assumes a  coefficient matrix $\bC^*\in\reals^{m\times q}$ and a coefficient vector $\gamma^*\in\real^p$, and for each $1\leq i\leq n$, the $i$-th univariate response $y_i$ is generated from the matrix predictor $\Xb_i\in \real^{m \times q}$ and $\bz_i\in\reals^{p}$ by
\begin{equation}\label{eq:model1}
    y_i=\langle \Xb_i,\bC^*\rangle+\langle \bz_i,\gamma^*\rangle+\epsilon_i,
\end{equation}
where $\langle \Xb_i,\bC^*\rangle$ is defined as the trace of $\bC^{*T}\Xb_i$, $\langle \bz_i,\gamma^*\rangle$ is defined as $\bz_i^T\gamma^*$, and $\{\epsilon_i\}_{i=1}^n$ are random errors. We remark that the trace regression model in \cite{10.1214/10-AOS860} and \cite{elsener2018robust} can be considered as a special case of this model when $\gamma^*$ is zero.

In binary classification problems, the response variables are binary with value $1$ or $0$ to indicate the presence or absence of the target category. For such applications, we follow the matrix variate logistic regression model from \cite{HungWang2012} and \cite{li2021double} and assume that the response $y_i$ is a binary variable such that
\begin{equation}\label{eq:model2}
\mathrm{logit}\left(P(y_i=1\mid \Xb_i)\right)=\log\left(\frac{P(y_i=1|\Xb_i)}{1-P(y_i=1|\Xb_i)}\right)=\langle \Xb_i,\bC^*\rangle+\langle \bz_i,\gamma^*\rangle,
\end{equation}
and the explicit formula is
\begin{equation}\label{eq:model2-2}
    y_i\sim \mathrm{Bernoulli}(p_i), \,\,\mbox{where}\,\, p_i=\frac{e^{\theta_i}}{1+e^{\theta_i}}\mbox{ and } \theta_i=\langle \Xb_i,\bC^*\rangle+\langle \bz_i,\gamma^*\rangle.
\end{equation}
The matrix variate regression model \eqref{eq:model1} and the matrix variate logistic regression model \eqref{eq:model2} are  generalizations of the regular regression model and the regular logistic regression model. The challenge is to estimate the coefficient matrix $\bC^*\in\reals^{m\times q}$ with a low-rank constraint and the coefficient vector $\gamma^*\in\real^p$, based on predictors $\{\bX_i,\bz_i\}_{i=1}^n\subset \reals^{m\times q}\times \reals^{p}$ and the associated responses $\{y_i\}_{i=1}^n$ generated from the matrix variate regression model \eqref{eq:model1} or the matrix variate logistic regression model \eqref{eq:model2}.


Since $\bC^*$ is assumed to have a low-rank structure (or can be approximated by a low-rank structure), we propose an estimator under the rank constraint as follows:
\begin{align}
(\hat{\Cb},\hat{\gamma})=\argmin_{\gamma\in\real^p, \Cb\in\real^{m\times q},\rank(\Cb)=r}\sum_{i=1}^n l(y_i,\langle \Xb_i,\Cb\rangle+\gamma^T\zb_i)+\lambda P(\bC,\gamma),
\label{eq:optimization}
\end{align}
where $l$ is a loss function that measures the difference between the observed response $y_i$ and predicted response $\langle \Xb_i,\Cb\rangle+\gamma^T\zb_i$. We apply two different loss functions to the matrix variate regression model and the robust matrix variate regression model \eqref{eq:model1} and another loss function for the matrix variate logistic regression model \eqref{eq:model2}. In summary, we consider the following three models, where the first two models follow Equation \eqref{eq:model1} and the third model follows Equation \eqref{eq:model2}.
\begin{itemize}
\item  \textbf{The ordinary matrix variate regression model}. Here we apply \eqref{eq:model1} and assume that $\epsilon_i$ follows a sub-Gaussian distribution with mean zero and sub-Gaussian parameter $\sigma$ \citep[Definition 2.2]{wainwright_2019}. For this model, we apply the least squares loss function
$
l(y_i,f_i)=(y_i-f_i)^2.
$
\item \textbf{The robust matrix variate regression model}. Here we still apply \eqref{eq:model1}, but assume that the errors $\{\epsilon_i\}_{i=1}^n$ are sampled from a zero-mean (and possibly heavy-tail) distribution. Then we apply Huber’s loss function \citep{10.1214/aoms/1177703732}:
$
l(y_i,f_i)=\rho_{\alpha}\left(|y_i-f_i|\right),
$
where
\begin{equation}
\rho_{\alpha}(t) = \begin{cases}
\frac12 t^2 &\text{if $t\leq \alpha$}\\
\alpha (|t|-\alpha/2)&\text{if $t > \alpha$}
\end{cases}
\label{def_Huber}
\end{equation}
and $\alpha > 0$ is a tuning parameter. We remark that the Huber loss function $\rho_{\alpha}(t)$ in \eqref{def_Huber} is convex, differentiable, and when $\alpha$ goes to $\infty$, $\rho_{\alpha}(t)$
 converges to $t^2/2$. In addition, for a fixed $\alpha > 0$, $\rho'_{\alpha}(t)$, $\rho''_{\alpha}(t)$ are
all bounded. We use the convexity to prove Theorem \ref{thm:asymptotics} in Section~\ref{sec:consistency}. Intuitively, this implies that the impact of outlying observation $y_i$ will be
controlled and thus the corresponding statistical procedure of parameter estimation will be robust \citep{zhang2019robustness}.

\item \textbf{The matrix variate logistic regression model}. Here we use Equation \eqref{eq:model2} and apply the logit loss function for logistic regression \[l(y_i,f_i)=\log\left(1+\exp(f_i)\right)-y_if_i.\]
\end{itemize}
The penalty function $P$ in $F(\Cb,\gamma)$ of Equation \eqref{eq:optimization} can be chosen according to specific applications and the prior knowledge of the properties of $\bC^*$ and $\gamma^*$. For example, if $\bC^*$ is known to be elementwisely sparse, we may apply $P(\bC,\gamma)=\|\bC\|_1$. In this work, we leave the choice open for the theoretical analysis and use $P(\bC,\gamma)=\|\bC\|_1$ in our experiments. 

\section{Efficient Parameter Estimation Algorithms}

In this section, we develop an efficient algorithm for solving the  optimization problem in Equation \eqref{eq:optimization}. We derive
an alternating projected gradient descent algorithm that updates  $\bC$ and $\gamma$ alternatively. To incorporate the low-rank constraint on $\bC$, we apply the technique of projected gradient descent \citep{frank1956algorithm,bolduc2017projected}. 
Specifically, the general formula of the alternating projected gradient descent algorithm can be explicitly written as
\begin{align}
\Cb^{(\iter+1)}&=P_{r}\left(\Cb^{(\iter)}-\alpha^{(\iter)}_1 \frac{\partial}{\partial \Cb}F(\Cb,\gamma)\Big|_{\Cb=\Cb^{(\iter)},\gamma=\gamma^{(\iter)}}\right)\label{eq:update_C}\\
\gamma^{(\iter+1)}&=\gamma^{(\iter)}-\alpha^{(\iter)}_2 \frac{\partial}{\partial \gamma}F(\Cb,\gamma)\Big|_{\Cb=\Cb^{(\iter+1)},\gamma=\gamma^{(\iter)}},\label{eq:update_gamma}
\end{align}
where $\alpha^{(\iter)}_1$ and $\alpha^{(\iter)}_2$ are the update step sizes and $P_r$ is the projection to the nearest matrix of rank $r$, which can be obtained by the singular value decomposition. The partial derivatives $\frac{\partial}{\partial \Cb}F(\Cb,\gamma)$ and $\frac{\partial}{\partial \gamma}F(\Cb,\gamma)$ have explicit formulas as follows:
\begin{itemize}
    \item The ordinary matrix variate regression setting:
    \begin{align}
           \frac{\partial}{\partial \bC} F(\bC,\gamma)&=\frac{1}{n}\sum_{i=1}^n\bX_i(y_i-\langle \bX_i,\bC\rangle-\gamma^T\bz_i)+\lambda \frac{\partial}{\partial \bC} P(\bC,\gamma),\label{eq:F-derivative-C1} \\
    \frac{\partial}{\partial \gamma} F(\bC,\gamma)&=\frac{1}{n}\sum_{i=1}^n\bz_i(y_i-\langle \bX_i,\bC\rangle-\gamma^T\bz_i)+\lambda \frac{\partial}{\partial \gamma} P(\bC,\gamma)\label{eq:F-derivative-gamma1}
    \end{align}
    \item The robust matrix variate regression setting:
    \begin{align}
    \frac{\partial}{\partial \bC} F(\bC,\gamma)&=\frac{1}{n}\sum_{i=1}^n\bX_i\rho'_\alpha(y_i-\langle \bX_i,\bC\rangle-\gamma^T\bz_i)+\lambda \frac{\partial}{\partial \bC} P(\bC,\gamma),\label{eq:F-derivative-C2} \\
    \frac{\partial}{\partial \gamma} F(\bC,\gamma)&=\frac{1}{n}\sum_{i=1}^n\bz_i\rho'_\alpha(y_i-\langle \bX_i,\bC\rangle-\gamma^T\bz_i)+\lambda \frac{\partial}{\partial \gamma} P(\bC,\gamma)\label{eq:F-derivative-gamma2}
\end{align}

    \item The logistic matrix variate regression setting:
     \begin{align}
    \frac{\partial}{\partial \bC} F(\bC,\gamma)&=\frac{1}{n}\sum_{i=1}^n\bX_i\left(\frac{e^{\langle \bX_i,\bC\rangle+\gamma^T\bz_i}}{1+e^{\langle \bX_i,\bC\rangle+\gamma^T\bz_i}}-y_i\right)+\lambda \frac{\partial}{\partial \bC} P(\bC,\gamma),\label{eq:F-derivative-C3} \\
    \frac{\partial}{\partial \gamma} F(\bC,\gamma)&=\frac{1}{n}\sum_{i=1}^n\bz_i\left(\frac{e^{\langle \bX_i,\bC\rangle+\gamma^T\bz_i}}{1+e^{\langle \bX_i,\bC\rangle+\gamma^T\bz_i}}-y_i\right)+\lambda \frac{\partial}{\partial \gamma} P(\bC,\gamma)\label{eq:F-derivative-gamma3}
\end{align}
\end{itemize}

\RestyleAlgo{ruled}

\SetKwComment{Comment}{/* }{ */}
\begin{algorithm}
\caption{Regularized low-rank matrix variate regression}
\label{alg:matrix-reg-log}
\begin{flushleft}
{\bf Input:}  The samples $\{(y_i,\bX_i,\bz_i), i=1,\cdots,n\}$, initial $\bC^{(0)}$ and $\gamma^{(0)}$. \\
{\bf Output:} Estimated coefficients $\hat{\bC}$ and $\hat{\gamma}$.\\
{\bf Steps:}\\
{\bf 1:} Set step $\iter=0$.\\
{\bf 2:} Let $\bC^{(\iter+1)}=P_r(\bC^{(\iter)}-\alpha^{(\iter)}_1 \partial_\bC F(\bC^{(\iter)},\gamma^{(\iter)}))$ as in \eqref{eq:update_C}, where $\frac{\partial}{\partial \bC} F(\bC,\gamma)$ is defined in \eqref{eq:F-derivative-gamma1}, \eqref{eq:F-derivative-gamma2} and \eqref{eq:F-derivative-gamma3}, respectively, and $\alpha^{(\iter)}_1$ is chosen by a line search. \\
{\bf 3:} Let $\gamma^{(\iter+1)}=\gamma^{(\iter)}-\alpha^{(\iter)}_2\partial_\gamma F(\bC^{(\iter+1)},\gamma^{(\iter)})$ as in \eqref{eq:update_gamma}, where $\frac{\partial}{\partial \gamma} F(\bC,\gamma)$ is defined in \eqref{eq:F-derivative-gamma1}, \eqref{eq:F-derivative-gamma2} and \eqref{eq:F-derivative-gamma3}, respectively, and $\alpha^{(\iter)}_2$ is chosen by a line search. \\
{\bf 4:} Let $\bC^{(\iter+1)}=P_r(\bC^{(\iter+1)})$ where $P_r(\cdot)$ is the projection to the nearest matrix of rank $r$.\\
{\bf 5:} Set $\iter=\iter+1$.\\
{\bf 6:} Repeat steps 2-5 until $\|F(\bC^{(\iter)},\gamma^{(\iter)})-F(\bC^{(\iter-1)},\gamma^{(\iter-1)})\|_F \leq \eps_n$ where
$\eps_n$ is a pre-specified threshold, or the number of iterations exceeds the upper limit: $\iter > N^{(\max)}$.\\
{\bf Output:}  Set $\hat{\bC}=\bC^{(\iter)},\hat{\gamma}=\gamma^{(\iter)}$.
\end{flushleft}
\end{algorithm}

The update step size $\alpha$ is chosen by a line search \citep{wright1999numerical}. 
In particular, we apply the backtracking line search in Algorithm~\ref{alg:matrix-reg-log}, a search scheme based on the Armijo–Goldstein condition \citep{armijo1966minimization}, to find an optimal step size $\alpha_k$ as the direct line search is computationally expensive for our objective function. Starting with $\alpha=1$, we repeat $\alpha:=\beta \alpha$ with a fixed $\beta\in (0,1)$ until the objective value is smaller than the previous iteration, i.e.,
\begin{align}
    F(\bC^{(\iter+1)},\gamma^{(\iter)})&\leq F(\bC^{(\iter)},\gamma^{(\iter)})\nonumber\\ F(\bC^{(\iter+1)},\gamma^{(\iter+1)})&\leq F(\bC^{(\iter)},\gamma^{(\iter+1)}).
    \label{line_search_ineq}
\end{align}
 The parameter $\beta$ is often chosen to be between 0.1 and 0.8 in practice~\citep{boyd2004convex} and a smaller $\beta$ associated with a more crude search. In our simulations, we let $\beta=0.5$ in both step 2 and step 3.



The complete algorithm is summarized as Algorithm~\ref{alg:matrix-reg-log}. We remark that the line search rule in Algorithm~\ref{alg:matrix-reg-log} is well-defined, that is, for sufficiently small $\alpha^{(\iter)}_1$ and $\alpha^{(\iter)}_2$, we have nonincreasing objective values in \eqref{line_search_ineq} for $\bC^{(\iter+1)}$ and $\gamma^{(\iter+1)}$ defined in \eqref{eq:update_C} and \eqref{eq:update_gamma}.
Following the line search rule, the objective value is  monotonically nonincreasing, and  as a result, the algorithm is numerically stable and the convergence of the objective value $F(\bC^{(\iter)},\gamma^{(\iter)})$ is guaranteed. Therefore, the stopping criterion in Algorithm~\ref{alg:matrix-reg-log} is well-defined. This convergence property of  Algorithm~\ref{alg:matrix-reg-log} is summarized in Proposition~\ref{prop:convergence} with its proof provided in the Appendix. It states that the functional value converges and the accumulation points of the sequence $(\bC^{(\iter)},\gamma^{(\iter)})$ must be stationary points of the objective function.

\begin{proposition}\label{prop:convergence}
(a) The functional values $F(\bC^{(\iter)},\gamma^{(\iter)})$ are nonincreasing, and converge if $F$ is bounded from below.
(b) Each accumulation point of the sequence $(\bC^{(\iter)}, \gamma^{(\iter)})$ is a stationary point of the objective function $F(\bC,\gamma)$.
\end{proposition}

Algorithm~\ref{alg:matrix-reg-log} requires the starting values of parameters $(\bC^{(0)},\gamma^{(0)})$. In the implementation, we let $\bC^{(0)}$ be a zero matrix and $\gamma^{(0)}$ be a zero vector, which works well in the simulations in Section~\ref{sec:simulation}. We discuss alternative initialization methods  with theoretical guarantees in Section~\ref{sec:convergence2}.
As for computational costs, the calculation of  remark that the calculation of $\frac{\partial}{\partial \gamma} F(\bC,\gamma)$ and $\frac{\partial}{\partial \bC} F(\bC,\gamma)$ in each iteration cost $O(n(mq+p))$, and the projection operator in \eqref{eq:update_C} has a computational cost of $O(rmq)$. As a result, the computational cost per iteration is in the order of $O((n+r)mq+np)$, which is comparable to the spectral regularized regression estimator (SRRE) method \citep{zhou2014regularized} and the double fused Lasso regularized
regression method \citep{li2021double}.

\section{Asymptotic Theory}\label{sec:consistency}
This section is devoted to analyzing the statistical consistency property of Algorithm~\ref{alg:matrix-reg-log}. First, in Section~\ref{sec:mainresults} we present our main results on the consistency of the proposed estimator in Theorem~\ref{thm:asymptotics} and the convergence property when the algorithm is well-initialized in Theorem~\ref{prop:convergence2}. Then we describe a sketch of the proof in Section~\ref{sec:skecth}, and leave the technical proofs to the Appendix.

\subsection{Consistency of the proposed estimator}\label{sec:mainresults}
We first present a few conditions for establishing the model consistency of the proposed estimator in \eqref{eq:optimization}.  

\textbf{Assumption 1}: There exists a positive constant $C_1>0$ such that
\[
\frac{1}{n}\Big\|\sum_{i=1}^n \mathrm{vec}(\bX_i,\bz_i)\mathrm{vec}(\bX_i,\bz_i)^T\Big\|\leq C_1,
\]
where $\mathrm{vec}(\bX_i,\bz_i)\in\reals^{qm+p}$ is a vector consisting of the elements in $\bX_i\in\reals^{m\times q}$ and $\bz_i\in\reals^p$.

\textbf{Assumption 2}: For $\bH(\bC,\gamma)\in\reals^{(qm+p)\times (qm+p)}$  defined by
\begin{equation}\label{eq:hessian}
\bH(\bC,\gamma)=\sum_{i=1}^nw_{2,i}\mathrm{vec}(\bX_i,\bz_i)\mathrm{vec}(\bX_i,\bz_i)^T,
\end{equation}
where
\begin{align*}
w_{2,i}=\begin{cases}2,\,\,\text{for the ordinary matrix variate regression model,}\\
I\Big(|y_i-\langle\bX_i,\bC\rangle-\bz_i^T\gamma|<\alpha\Big),\,\,\text{for the robust matrix variate regression model,}\\
e^{\langle\bX_i,\bC\rangle+\bz_i^T\gamma}/(1+e^{\langle\bX_i,\bC\rangle+\bz_i^T\gamma})^2,\,\,\text{for the logistic matrix variate regression model},\end{cases}
\end{align*}
and $\frac{1}{n}\bH(\bC,\gamma)$ is positive definite with eigenvalues bounded from below for all  $(\bC,\gamma)$ in a neighborhood of $(\bC^*,\gamma^*)$. Specifically, there exists positive constants $C_2>0$ and $c_0\leq \sigma_r(\bC^*)/2$ such that 
$
\frac{1}{n}\lambda_{\min}\Big(\bH(\bC,\gamma)\Big)\geq C_2
$ and for all $(\bC,\gamma)$ such that $$\dist((\bC,\gamma),(\bC^*,\gamma^*))=\sqrt{\|\bC-\bC^*\|_F^2+\|\gamma-\gamma^*\|^2}\leq c_0.$$

Assumptions 1 and 2 can be considered as the generalized version of Restricted Isometry condition in \cite{doi:10.1137/070697835} and are comparable to  \cite[Conditions 1,2,5,6]{li2021double}. In particular, Assumption 1 ensures that the sensing vectors $\mathrm{vec}(\bX_i,\bz_i)\in\reals^{qm+p}$ are not the same or in a similar direction, and it is satisfied if the sensing vectors are sampled from a distribution that is relatively uniformly distributed among all directions. Assumption 2 ensures that the Hessian matrix of $\sum_{i=1}^n l(y_i,\langle \Xb_i,\Cb\rangle+\gamma^T\zb_i)$, $\bH$, is not a singular matrix.

\begin{remark}
Assumptions 1 and 2 are reasonable when $n>O(pm+q)$ and $(\bX_i,\bz_i)$ is sampled from a reasonable distribution that does not concentrate around certain directions. For example, if $\mathrm{vec}(\bX_i,\bz_i)$ are i.i.d. sampled from a distribution of $N(0,\bI)$, then the standard concentration of measure results~\citep{wainwright_2019} imply that for the matrix variate regression model, $\bH=2\sum_{i=1}^n \mathrm{vec}(\bX_i,\bz_i)\mathrm{vec}(\bX_i,\bz_i)^T$, and the standard concentration of measure results~\citep{wainwright_2019} imply that Assumption 1 holds with a high probability when $n=O(qm+p)$ and Assumption 2 holds for the regression model as well since $\bH=2\sum_{i=1}^n \mathrm{vec}(\bX_i,\bz_i)\mathrm{vec}(\bX_i,\bz_i)^T$. Assumption 2 holds for the logistic regression model as long as  $a_i=\langle\bX_i,\bC\rangle+\bz_i^T\gamma$ is bounded above for most indices $i$ as $w_{2,i}\geq 1/(2e^{a_i})$, and holds for the robust regression model if  $\alpha$ is not too small such that there are sufficient samples (e.g., more than $2(qm+p)$ samples) such that $|y_i-\langle\bX_i,\bC\rangle-\bz_i^T\gamma|<\alpha$.
\end{remark}

\textbf{Assumption 3} [Assumption on the noise for the matrix variate regression model and the robust matrix variate regression model]: For the matrix variate regression model, the error $\epsilon_i$'s in Equation \eqref{eq:model1} follow an independent and identically distributed (i.i.d.) zero-mean and sub-Gaussian distributions with zero mean and variance one, i.e., $\Expect (\epsilon_i )= 0$ and $\Var (\epsilon_i )=\sigma^2$ to ensure that the distribution of outliers is limited as this model is not designed to handle outliers. For the robust matrix variate regression model, the errors $\epsilon_i$'s follow i.i.d. zero-mean and possibly heavy-tailed distributions.

With Assumptions 1-3, we establish the consistency of the proposed  estimator \eqref{eq:optimization} in Theorem~\ref{thm:asymptotics} that shows for large $n$, our estimator converges to the true underlying solution $(\bC^*,\gamma^*)$. This result holds for general penalty $P$ and regularization parameter $\lambda$.  A sketch of the proof is presented in Section~\ref{sec:skecth} and the detailed proof is deferred to the Appendix.


\begin{theorem}[Consistency of the proposed estimator]\label{thm:asymptotics}
Under Assumptions 1-3, let
\begin{align*}
\sigma_0=\begin{cases}\sigma,\,\,\text{for the matrix variate regression model}\\1,\,\,\text{for the robust matrix variate regression model}\\1,\,\,\text{for the logistic matrix variate regression model},\end{cases}
\end{align*}
then for 
\begin{align}
n\geq C \left(6\frac{C_1\sigma_0\sqrt{(r(qm+p)}}{C_2}+3\sqrt{\frac{\lambda P(\bC^*,\gamma^*)}{C_2}}\right)^2\Bigg/c_0^2
 \label{eq:assumption_asymptotics}
\end{align}
and for all $t\geq 2$, we have the following upper bound on the estimation error of the proposed estimator \eqref{eq:optimization} with probability at least $1-C\exp(-Cn)-C\exp(-Ct(r(q+m)+p))$:
\begin{equation}\label{eq:bound_consistency}
\dist((\hat{\bC},\hat{\gamma}),(\bC^*,\gamma^*))
\leq 6\frac{C_1 t \sigma_0\sqrt{(r(q+m)+p)} }{C_2\sqrt{n}}+\sqrt{\frac{6\lambda P(\bC^*,\gamma^*)}{C_2n}}.
\end{equation}
\end{theorem}
\subsection{Discussion of Theorem~\ref{thm:asymptotics}}
\hfill

\textbf{Comparison of the convergence rate with existing works.}
The main result, Inequality \eqref{eq:bound_consistency}, shows that the estimation error is in the order of \begin{equation}\label{eq:error_order}O_P\left(\sqrt{\frac{r(q+m)+p}{n}}+\sqrt{\frac{\lambda P(\bC^*,\gamma^*)}{n}}\right).\end{equation} In comparison, the rate in the existing work \citep{li2021double} is $$O_P\left(\sqrt{\frac{qm+p}{n}}+\sqrt{\frac{\lambda P(\bC^*,\gamma^*)}{n}}\right).$$ 
Our result improves the factor $mq + p$ to $r(q+m)+p$, which is a significant improvement when the rank $r$ is smaller than $q$ and $m$. This improvement can be explained by the fact that by fixing the rank constraint, our estimator has $r(q+m)+p$ parameters which are fewer than the $qm+p$ parameters in the estimator of \cite{li2021double}. With fewer parameters to estimate, our estimator achieves a better convergence rate, and this improvement is also clear from our simulation experiments. $O_P$ and $o_P$ represent the standard big and small O notation in probability.

\textbf{Comparison with minimax rate.}  We derive our minimax analysis in Theorem \ref{prop:convergence2}, which shows that under the setting that $\bC^*$ is low-rank, the minimax rate of estimating $(\bC,\gamma)$ is in the order of $O_P\left(\sqrt{\frac{r(q+m)+p}{n}}\right)$. In comparison, our rate in Theorem~\ref{thm:asymptotics}  achieves this rate. In literature, \citet[Remark 10]{luo2020recursive} shows that the minimax rate is in the order of $O_P\left(\sqrt{\frac{r(q+m)}{n}}\right)$ when the vector covariates $\gamma$ does not exist and $p=0$ and can be considered as a special case of our result. In the future, we plan to derive an improved minimax rate of our model under the setting that $\bC^*$ is both sparse and low-rank based on the technique of \cite{she2021analysis}, and show that our estimator with a well-chosen $\lambda\neq 0$ and $P(\bC,\gamma)=\|\bC\|_1$ achieves  the improved rate. 

\textbf{MLE argument when $\lambda=0$.} In addition, the adaption of the usual arguments of MLE to manifold optimization can be applied when $\lambda=0$, in which we may show that $\sqrt{n}(\tilde{\bC}-\bC,\tilde{\gamma}-\gamma)$ converges to a Gaussian-like distribution (the result is similar to \citet[Theorem 3.1]{HungWang2012}), where the covariance/Fisher information matrix can be obtained following the technique of \cite{le1990maximum} and \cite{6418045}. However, we skip the rigorous statement here as the we use $\lambda>0$ in practice.

\textbf{A generic choice of $P$.}  In this work, we leave the choice of the penalty $P$ open for the theoretical analysis. However, there are natural choices of $P$ in certain applications. For example, when we have the prior knowledge that $\bC^*$ is sparse, then we may let $P(\bC,\gamma)=\|\bC\|_1$.  

\textbf{Choice of rank $r$} The choice of $r$ has a large impact on the solution. In fact, the objective value of \eqref{eq:optimization} is nonincreasing as $r$ increases, since the set of matrix of rank $r_2$ contains the set of matrices of rank $r_1$ when $r_2>r_1$. We expect that a choice of $r$ can be obtained using either cross validation, or the elbow method \citep{choi2017selecting} that chooses the rank such that a larger rank doesn't give much smaller objective value.

\subsection{Convergence of the proposed algorithm}
In this section, we show that the output of Algorithm~\ref{alg:matrix-reg-log} is linearly convergent and sufficiently close to the ground truth with a good initialization and an additional assumption on the second order information of the ``Hessian matrix'' $\bH$ in Assumption 2.

\textbf{Assumption 4}: There exists a positive constant $C_3>0$ such that $\frac{1}{n}\Big\|\bH(\bC,\gamma)\Big\|\leq C_3$ for all $(\bC,\gamma)$ such that $\dist((\bC,\gamma),(\bC^*,\gamma^*))=\sqrt{\|\bC-\bC^*\|_F^2+\|\gamma-\gamma^*\|^2}\leq c_0.$
Combining it with Assumption 2, it suggests that $\bH$ is a well-conditioned matrix. Hence, Assumption 4 can be considered as the generalized version of the Restricted Isometry condition in \cite{doi:10.1137/070697835} and comparable to  \citet[Conditions 2,6]{li2021double}. 

Assumptions 4 is reasonable when $n>O(pm+q)$ and $(\bX_i,\bz_i)$ is sampled from a reasonable distribution that does not concentrate around certain directions. For example,  if $\mathrm{vec}(\bX_i,\bz_i)$ are i.i.d. sampled from a distribution of $N(0,\bI)$, then the standard concentration of measure results~\citep{wainwright_2019} imply that Assumption 2 holds for all three models with a high probability, since $w_{2,i}$ are bounded above.
With Assumption 4, we have the following result showing the convergence of Algorithm~\ref{alg:matrix-reg-log}, with its proof deferred to the Appendix. It shows that with a good initialization, all accumulation points have estimation errors converging to zero as $n\rightarrow\infty$.

\begin{theorem}[Convergence of the proposed algorithm]\label{prop:convergence2}
Under Assumptions 1-4, let $c_0'=\min\left(c_0,\frac{C_2}{4C_TC_3}\right)$, 
the initialization $(\bC^{(0)},\gamma^{(0)})$ is ``good'' in the sense that it is close to $(\bC^*,\gamma^*)$, the objective value $F(\bC^{(0)},\gamma^{(0)})$ is not large: 
\begin{align}\nonumber
\dist(({\bC}^{(0)},{\gamma}^{(0)}),(\bC^*,\gamma^*))&\leq c_0',\\
F(\bC^{(0)},\gamma^{(0)})-\sum_{i=1}^n l(y_i,\langle \Xb_i,\Cb^*\rangle+\gamma^{*T}\zb_i)&\leq \frac{C_2n}{2}c_0'^2,\label{eq:convergence_init}
\end{align}
and $n$ is large: $\sqrt{n}>\frac{4C_TC_1\sigma_0\sqrt{qm+p} }{C_2}.$
Then for all $t\geq 2$, with a probability at least $1-C\exp(-Cn)-C\exp(-t(r(q+m)+p))$, all accumulation points of the iterates $\{(\bC^{(\iter)}, \gamma^{(\iter)})\}_{\iter\geq 1}$, denoted by $(\tilde{\bC},\tilde{\gamma})$, have small estimation errors:
\begin{equation}\label{eq:stationary}
\dist((\tilde{\bC},\tilde{\gamma}),(\bC^*,\gamma^*))\leq \frac{4C_1t\sigma_0\sqrt{r(q+m)+p} }{C_2\sqrt{n}}+\frac{4\lambda C_{partial}}{nC_2},
\end{equation}
 where $C_{partial}$ is the upper bound of the magnitude of the derivative of the penalty function within a neighborhood of the ground truth:
$$C_{partial}=\max_{\{\dist((\bC,\gamma),(\bC^*,\gamma^*))\leq c_0'\}}\sqrt{\Big\|\frac{\partial}{\partial \bC}P(\bC,\gamma)\Big\|_F^2+\Big\|\frac{\partial}{\partial \gamma}P(\bC,\gamma)\Big\|^2}.$$
\end{theorem}
\begin{remark}
Unlike Theorem~\ref{thm:asymptotics}, Theorem~\ref{prop:convergence2} does not depend on the convexity of the loss function $l$. As a result, it applies to various choices of $l$, including popular loss function functions in robust statistics such as redescending $\psi$’s, Hampel’s loss, or Tukey’s bisquare, etc.~\citep{10.1214/aoms/1177703732,maronna2018robust,she2017robust,huang2020robust} that can detect outliers with moderate or high leverages. 
\end{remark}

\subsection{Minimax rate}\label{sec::minimax}
In this section, we prove the minimax lower bounds
showing that the rates attained by our estimator are optimal, based on the arguments in \cite{tsybakov2008introduction}. We consider the
class of parameters
\[
\mathcal{A}(r,a)=\{(\bC,\gamma)\in\reals^{q\times p}\times \reals^p: \mathrm{rank}(\bC)\leq r, \|\bC\|_F^2+\|\gamma\|^2\leq 1\}.
\]
We will need the following assumptions:\\
\textbf{Assumption 5:} There exists $C_{upper}>0$ such that for all $(\bC,\gamma)\in \mathcal{A}(r,a)$,
\[
 \sum_{i=1}^n\Big(\langle\bX_i,\bC\rangle+\gamma^T\bz_i\Big)^2\leq  C_{upper}n\|\mathrm{vec}(\bC,\gamma)\|^2.
\]
\textbf{Assumption 6:} There exists $c_{\epsilon}>0$ such that for all $x\in\reals$,
$
KL(P_{\epsilon,0},P_{\epsilon,x})\leq c_{\epsilon}x^2,
$
where $P_{\epsilon,x}$ represent the distribution of  $\epsilon_i+x$ for the regression model and the robust regression model, and $P_{\epsilon,x}$ represent the Bernoulli distribution with parameter $x$ for the logistic regression model. In addition, $KL$ represents the Kullback-Leibler divergence: For distributions $P$ and $Q$ of a continuous random variable with probability density functions $p(x)$ and $q(x)$, it is defined to be the integral $KL(P,Q)=\int_{x}p(x)\log(p(x)/q(x))\di x$.

Assumption 5 is less restrictive of Assumption 1 as it only need to be true for all $(\bC,\gamma)\in\mathcal{A}(r,a)$, and Assumption 6 holds under the logistic regression model, and under the regression and robust regression models, Assumption 6 holds for  zero-mean, symmetric distributions with tails decaying not faster than Gaussian, including Gaussian distribution, exponential distribution, Cauchy distribution, and Student's t distribution.
Now we are ready to state our main result on the minimax rate. Combining it with convergence rate in Theorem~\ref{thm:asymptotics}, it implies that the rate of our estimator is optimal. 
\begin{theorem}\label{thm:minimax}
Assuming that $n\geq r(m+q)+p$, then for any $0<\beta<\frac{2^{(r(m+q)+p)/4}}{1+2^{(r(m+q)+p)/4}}$, there exists $c_0>0$ depending on $\beta, c_{\epsilon}$ and $C_{upper}$ such that
\[
\inf_{\hat{\bC},\hat{\gamma}}\sup_{(\bC^*,\gamma^*)\in \mathcal{A}(r,a)}\Pr\left(\dist\Big((\bC^*,\gamma^*),(\hat{\bC},\hat{\gamma})\Big)\geq c_0 \sqrt{\frac{r(m+q)+p}{n}} \right)\geq \beta.
\]
\end{theorem}
\begin{proof}
WLOG assume that $m\geq q$. Let $\theta=(\bC,\gamma)$, and 
\[
\mathcal{C}=\Bigg\{(\bC,\gamma):\bC=[\bC',0]\,\, \text{where}\,\, \bC'\in\reals^{m\times r}\,\, \text{and}\,\, 0\in\reals^{m\times (q-r)},  \bC'_{ij}\in\{s,-s\}, \gamma_k\in\{s,-s\}\Bigg\},
\]
where $s=c_0\sqrt{\frac{r(m+q)+p}{n}}$. Then we have $|\mathcal{C}|=2^{rm+p}$, and for any $(\bC_1,\gamma_1), (\bC_2,\gamma_2)\in\mathcal{C}$, $\dist((\bC_1,\gamma_1), (\bC_2,\gamma_2))\geq 2s$. In addition, for any $(\bC,\gamma)\in\mathcal{C}$ and $P_0$ represents the model when $\bC=0$ and $\gamma=0$, 
\[
K(P_{0},P_{(\bC,\gamma)})\leq c_{\epsilon}\sum_{i=1}^n\Big(\langle\bX_i,\bC\rangle+\gamma^T\bz_i\Big)^2\leq c_{\epsilon}C_{upper}(\|\bC\|_F^2+\|\gamma\|^2)\leq c_{\epsilon}C_{upper} (r(m+q)+p)s^2.
\]
Applying \cite[Theorem 2.5]{tsybakov2008introduction} and note that $\log|\mathcal{C}|=2^{rm+p}=(rm+p)\log 2$, we may choose $
\alpha={2c_{\epsilon}C_{upper}}{s^2}/\log 2$, and $\alpha$ can be sufficiently small by choosing $c_0$ to be small. The rest of the proof following applying \cite[Theorem 2.5]{tsybakov2008introduction}.
\end{proof}
\subsection{Sketch of the Proof of Theorem~\ref{thm:asymptotics}}\label{sec:skecth}
We start with the intuition of the proof with a function $f: \reals^p\rightarrow\reals$. To show that $f$ has a local minimizer around $\bx^*$, it is sufficient to show that the gradient $\nabla f(\bx^*)\approx 0$ and the Hessian matrix of $f(\bx)$, $\bH(\bx)$, is positive definite with eigenvalues strictly larger than some constant $c>0$. The intuition of the proof follows from the Taylor expansion that \begin{align}\label{eq:taylor}f(\bx)\approx& f(\bx^*)+(\bx-\bx^*)^T\nabla f(\bx^*) + \frac{1}{2}(\bx-\bx^*)^T H(\bx^*)(\bx-\bx^*)\\\geq& f(\bx^*)+(\bx-\bx^*)^T\nabla f(\bx^*) + \frac{c}{2}\|\bx-\bx^*\|^2.\nonumber\end{align} As a result, there is local minimizer in the neighbor of $\bx^*$ with radius $\frac{\|\nabla f(\bx^*)\|}{c}$, i.e., $B\left(\bx^*,\frac{\|\nabla f(\bx^*)\|}{c}\right)$.
To extend this proof to \eqref{eq:optimization}, the main obstacle is the nonlinear constraint in the optimization problem. To address this issue, we consider the constraint set in \eqref{eq:optimization} as a manifold and generalize the ``second-order Taylor expansion'' in \eqref{eq:taylor} to the function defined on a manifold. With this generalized Taylor expansion, a similar strategy can be applied to prove that the minimizer of \eqref{eq:optimization} is close to $(\bC^*,\gamma^*)$.

To analyze functions defined on manifolds, we introduce a few additional notations. We assume a manifold $\calM\subset\reals^p$ and a function $f:\reals^p\rightarrow \reals$, and investigate $f(\bx)$ for $\bx\in B(\bx^*,r)\cap \calM$, i.e., a local neighborhood of $\bx^*$ on the manifold $\calM$. We denote the first and second derivatives of $f(\bx)$ by $\nabla f(\bx)\in\reals^{p}$ and $\bH(\bx)\in\reals^{p\times p}$ respectively, the tangent plane of $\calM$ at $\bx^*$ by $T_{\bx^*}(\calM)$, and let $\Pi_{T_{\bx^*}(\calM)}$ and $\Pi_{T_{\bx^*}(\calM),\perp}$ be the projectors to $T_{\bx^*}(\calM)$ and its orthogonal subspace respectively. These definitions are visualized in Figure~\ref{fig:manifold}.

\begin{figure}[ht]
  \centering
  \includegraphics[width=0.7\linewidth]{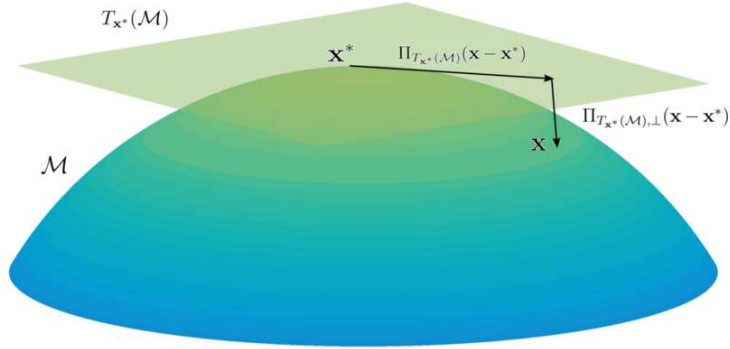}
  \caption{A visualization of the manifold $\calM$, two points $\bx,\bx^*\in\calM$, the tangent space $T_{\bx^*}(\calM)$, and the projectors $\Pi_{T_{\bx^*}(\calM)}$ and $\Pi_{T_{\bx^*}(\calM),\perp}$. }
  \label{fig:manifold}
\end{figure}

Then, we say that a manifold $\calM$ is curved with parameter $(c_0,C_T)$ at $\bx^*\in\calM$, if for any $\bx\in B(\bx^*,c_0)\cap\calM$, we have $\|\Pi_{T_{\bx^*}(\calM),\perp}(\bx-\bx^*)\|\leq C_T\|\Pi_{T_{\bx^*}(\calM)}(\bx-\bx^*)\|^2$. Intuitively, it means that the projection of $\bx-\bx^*$ to the tangent space $T_{\bx^*}(\calM)$ has a larger magnitude than the projection to the orthogonal subspace of the tangent space (see Figure~\ref{fig:manifold}).We remark that larger $C_T$ means that the manifold $\calM$ is more ``curved'' around $\bx^*$. Then,  Lemma~\ref{lemma:general_manifold} establishes the lower bound of $f$ based on the local properties such as the first and second derivatives of $f$ at $\bx^*$, the tangent space of $\calM$ around $\bx^*$, and the curvature parameters $(c_0,C_T)$.
\begin{lemma}\label{lemma:general_manifold}
Consider a $d$-dimensional manifold $\calM\subset\reals^p$ and a function $f: \reals^P\rightarrow\reals$, define
$
C_{H,1}= \min_{\bx\in B(\bx^*,c_0)} \lambda_{\min}(\bH(\bx)), $
and assume that $\calM$ is curved with parameter $(c_0,C_T)$ at $\bx^*$ and $4C_{H,1}\geq C_T\|\Pi_{T_{\bx^*}(\calM),\perp}\nabla f(\bx^*)\|$, then we have the following lower bound for any $\bx\in B(\bx^*,c_0)\cap\calM$:
\[
f(\bx)-f(\bx^*)\geq \frac{1}{2}b^2C_{H,1}- b \|\Pi_{T_{\bx^*}(\calM)}\nabla f(\bx^*)\|-C_Tb^2\|\Pi_{T_{\bx^*}(\calM),\perp}\nabla f(\bx^*)\|,
\]
where $b=\|\bx-\bx^*\|$.
\end{lemma}

Lemma~\ref{lemma:general_manifold} can be viewed as a generalization of Inequality \eqref{eq:taylor}: when $\calM=\reals^p$, then we have $C_T=0$, $\Pi_{T,\perp}=\emptyset$ and as a result, $\|\Pi_{T_{\bx^*}(\calM),\perp}\nabla f(\bx^*)\|=0$.
To apply Lemma~\ref{lemma:general_manifold} to our problem, we need to estimate the parameters  $c_0, C_T, C_{H,1},$  $\|\Pi_{T_{\bx^*}(\calM)}\nabla f(\bx^*)\|$, and $\|\Pi_{T_{\bx^*}(\calM),\perp}\nabla f(\bx^*)\|$ in the statement of Lemma~\ref{lemma:general_manifold}. In particular, $(c_0,C_T)$ depends on the manifold $\calM$ used in the optimization problem \eqref{eq:optimization}, which is \begin{equation}\calM=\{(\bC,\gamma)\in\reals^{q\times m}\times \reals^p: \rank(\bC)=r\}\label{eq:manifold}.\end{equation} By treating $\calM$ as the product space of $\reals^p$ and the manifold of low-rank matrices $\{\bC\in\reals^{q\times m}: \rank(\bC)=r\}$ and following the  tangent space  of the set of low-rank matrices in the literature \citep{Absil2015,10.5555/3291125.3309642}, we obtain the following lemma of ``curvedness'' parameters $(c_0,C_T)$ of $\calM$ at $(\bC^*,\gamma^*)$.

\begin{lemma}\label{lemma:curvature}
The manifold $\calM$ defined in \eqref{eq:manifold} is curved with parameter $(c_0,C_T)$ at $(\bC^*,\gamma^*)$, for any $c_0\leq \sigma_{r}(\bC^*)/2$ and $C_T=2/\sigma_{r}(\bC^*)$, where $\sigma_{r}(\bC^*)$ represents the $r$-th (i.e., the smallest) singular value of $\bC^*$.
\end{lemma}
In addition, the parameter $C_{H,1}$ in Lemma~\ref{lemma:general_manifold} can be estimated from Assumption 2, and the derivatives $\|\Pi_{T_{\bx^*}(\calM)}\nabla f(\bx^*)\|$ and $\|\Pi_{T_{\bx^*}(\calM),\perp}\nabla f(\bx^*)\|$ in Lemma~\ref{lemma:general_manifold} can be estimated from Assumptions 1 and 3. Then the proof of Theorem~\ref{thm:asymptotics} follows from  Lemma~\ref{lemma:general_manifold} and the intuition introduced at the beginning of this section, with technical details deferred to the Appendix.

\subsection{Discussion of Theorem~\ref{prop:convergence2}}\label{sec:convergence2}
\hfill

\textbf{Convergence.} The key result in this section, \eqref{eq:stationary}, shows that the algorithm achieves a similar estimation error as the result in the consistency result in Inequality \eqref{eq:bound_consistency}, with the term $\sqrt{\frac{6\lambda P(\bC^*,\gamma^*)}{C_2n}}$ replaced with $\frac{4\lambda C_{partial}}{nC_2}$. For many common penalty functions, $C_{partial}$ is bounded. For example, for the $\ell_1$ loss function that $P(\bC,\gamma)=\|\bC\|_1$ used in our simulations, we have $C_{partial}\leq \sqrt{qm}$.

\textbf{Condition on initialization in Inequality \eqref{eq:convergence_init}.} Theorem~\ref{prop:convergence2} makes two assumptions on the initialization in Inequality \eqref{eq:convergence_init}, where the LHS of the first condition is the distance between $(\bC^{(0)},\gamma^{(0)})$ and the true solution $(\bC^*,\gamma^*)$. While the second condition is more complicated, but it is satisfies when $\dist((\bC^{(0)},\gamma^{(0)}),(\bC^*,\gamma^*))=o_P(1)$ and $\lambda=o(1)$ as $n\rightarrow\infty$: note that the LHS is
\begin{align*}
&F(\bC^{(0)},\gamma^{(0)})-\sum_{i=1}^n l(y_i,\langle \Xb_i,\Cb^*\rangle+\gamma^{*T}\zb_i)\\=&\sum_{i=1}^n l(y_i,\langle \Xb_i,\Cb^{(0)}\rangle+\gamma^{(0)T}\zb_i)-\sum_{i=1}^n l(y_i,\langle \Xb_i,\Cb^*\rangle+\gamma^{*T}\zb_i)+\lambda P(\bC^{(0)},\gamma^{(0)}),
\end{align*}
and Assumption 4 implies that 
\[
\sum_{i=1}^n l(y_i,\langle \Xb_i,\Cb^{(0)}\rangle+\gamma^{(0)T}\zb_i)-\sum_{i=1}^n l(y_i,\langle \Xb_i,\Cb^*\rangle\leq nC_3 \dist((\bC^{(0)},\gamma^{(0)}),(\bC^*,\gamma^*))^2.
\]

\textbf{Initialization.} Theorem~\ref{prop:convergence2} requires an initialization that is within a neighborhood of the true solution of radius $O(1)$. While empirically we initialize $\bC^{(0)}$ and $\gamma^{(0)}$ as a zero matrix and a zero vector, it is possible to obtain initialization with theoretical guarantees. For example, for the regression setting, we may let the initialization to be the solution of the standard regression problem, i.e., the solution of $\argmin_{\bC,\gamma}\sum_{i=1}^n (y_i-\langle \Xb_i,\Cb\rangle-\gamma^T\zb_i)^2$, which has an initial estimation error in the order of $O_P\left(\sqrt{\frac{qm+p}{n}}\right)$. For the robust regression and logistic regression settings, we may solve $\argmin_{\bC,\gamma}\sum_{i=1}^n l(y_i,\langle \Xb_i,\Cb\rangle+\gamma^T\zb_i)$ as well. These are convex problems, and the standard asymptotic analysis shows that the estimation errors would converge to zero as $n\rightarrow\infty$, that is, the conditions in the inequality \eqref{eq:convergence_init} will be satisfied as $n\rightarrow\infty$.

\textbf{Convergence rate.} Unfortunately, it is difficult to obtain a general result of the convergence rate to $(\hat{\bC},\hat{\gamma})$ without assumptions on the penalty $P$. However, our proof implies that Algorithm~\ref{alg:matrix-reg-log} converges linearly to a neighborhood around the optimal solution $(\bC^*,\gamma^*)$, and empirically Algorithm~\ref{alg:matrix-reg-log} converges quickly and the convergence rate are linear in our simulations.

\section{Applications}\label{sec:simulation}

In this section, we carry out an extensive numerical study to investigate the empirical performance of our proposed methods on from simulated 2D shape data, brain signal electroencephalography (EEG) data, and leucorrhea data. We compare the proposed matrix variate regression models with the spectral regularized regression estimator (SRRE) proposed in \cite{zhou2014regularized}, and the robust low-rank regression matrix estimator (RLRME) proposed in \cite{elsener2018robust}.
For the logistic matrix variate regression model, we compare our estimator with the Simultaneous Differential Network analysis and Classification for Matrix-Variate data (SDNCMV) method \citep{hao2020simultaneous}, which used the
Constrained $l_1$-minimization for Inverse Matrix Estimation (CLIME) method \citep{cai2011constrained} and logistic regression with the elastic net penalty \citep{zou2005regularization}.
We put some figures and tables in the Supplement due to the space limit.

 Other similar regularized estimators based on SRRE include the double fused Lasso regularized matrix regression (DFMR) and double fused Lasso regularized matrix logistic regression (DFMLR) in \cite{li2021double}, but we skip the comparison with the DFMR and DFLMR estimators because they consider other conditions on the variate such as the sparsity in the difference of successive rows of the matrix variate and the sparsity in the vector variate, which are not considered in our settings.
 We also do not compare our estimator with \cite{li2018tucker} and \cite{zhou2013tensor} since these two works focus on the tensor setting rather than the matrix setting.
For the robust matrix variate regression model, we compare our estimator with the robust low-rank matrix estimator (RLRME)~\citep{elsener2018robust}, which can be considered as a robust version of the SRRE estimator.

\subsection{Simulation data analysis: matrix-variate regression}

In the first example, we generate a collection of 2D shapes as the true signals and the response is sampled as $y_i=\langle \bX_i,\bC^*\rangle+\langle\bz_i,\gamma^*\rangle+\epsilon_i$, where $\bX_i\in\reals^{64\times 64},\bz_i\in\reals^5$ are random with standard normal entries and $\epsilon_i$ is also standard normal. We set $\gamma^*=(1,1,1,1,1)^T$ in the following experiments. The true signal $\bC^*\in\reals^{64\times 64}$ is a 2D matrix whose entries are binary. We compare the performance of RMRE and SRRE with the various sample sizes at $n=300, 500, 700, 1000$. To tune the regularization parameter, \citet{zhou2014regularized} estimated the degree of freedom of SRRE and suggests to use Bayesian information criterion (BIC) to choose $\lambda$. In this experiment, we would rather use an independent validation set to tune the parameter $\lambda$ and a testing set to evaluate the error for both of the methods for fairness. The experiments are repeated 100 times and the mean with the standard deviation of the root mean square errors (RMSEs) of $\bC$ and $\gamma$ are present in Table~\ref{table:2d}. A comparison of the recovered signals of RMRE and SRRE at $n=500$ is shown in Fig.~\ref{fig:_triangle_circle_butterfly-robust}.

\begin{figure}[ht]
    \label{fig:2D-shapes}
    \centering
    \includegraphics[width=0.8\linewidth]{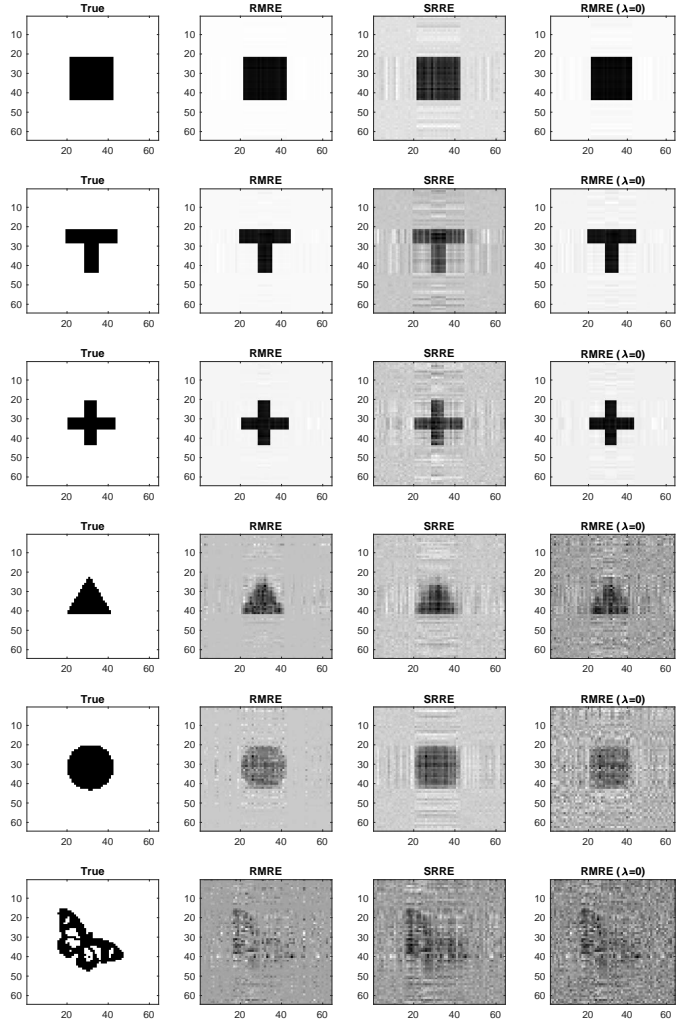}
    \caption{Comparison of the estimators RMRE and SRRE. The first column are the true signals, the fourth column are recovered signals of RMRE without regularization. The sample size is 500. }    
\end{figure}

As illustrated in Fig.~\ref{fig:_triangle_circle_butterfly-robust}, we evaluate the proposed method based on the recovered 2D shapes. First, RMRE outperforms SRRE for the square, T, and cross shapes, even without regularization ($\lambda=0$) it still estimates the signals well. RMRE and SRRE have similar performance in the triangle and circle shapes, while SRRE is slightly better for the butterfly. Second, the regularization of RMRE improved the estimation for the triangle and the circle shapes. Lastly, the recovered butterflies of both of RMRE and SRRE are barely recognized as the sample size $n=500$ is too small in this case.
The RMSEs (see  Table~\ref{table:2d}) of both $\bC$ and $\gamma$ of RMRE are much smaller than the RMSEs of SRRE with various values of $n$ for the square, T, cross shapes. As for the triangle and circle, both methods have similar performance when the sample size is small and RMRE has a better estimation as the sample size increases. RMRE performs slightly better than SRRE for the more complex shape, butterfly, when the sample size is taken as $n=1000$ and performs slightly worse when $n$ is smaller than 500. Moreover, the RMSEs of $\bC$ and $\gamma$ of RMRE are decreasing as the sample size is increasing, which verifies our theoretical analysis of the consistency in Section~\ref{sec:consistency}.

\begin{table}[ht]
\caption{Estimation errors for 2D shapes in the normal model . The means (standard deviation) of the RMSEs of $\hat \bC$ and $\hat \gamma$ with 100 repetitions of RMRE and SRRE are reported. }
\label{table:2d}
\resizebox{\linewidth}{!}{
\begin{tabular}{|c|c|c|c|c|c|c|}
\hline 
Shape & Para. & Method & $n=300$ & $n=500$ & $n=700$ & $n=1000$ \\ 
\hline \hline
\multirow{4}{*}{Square} & \multirow{2}{*}{$\bC$}  & RMRE & 0.0144(0.0018) & 0.0105(0.0011) & 0.0093(0.0009) & 0.0086(0.0009) \\
& & SRRE & 0.1934(0.0208) & 0.0648(0.0086) & 0.0300(0.0032)   & 0.0170(0.0011) \\ 
\cline{3-7}
& \multirow{2}{*}{$\gamma$} & RMRE & 0.0728(0.0245) & 0.0475(0.0153) & 0.0424(0.0140)  & 0.0332(0.0103) \\
& & SRRE & 0.6465(0.2352) & 0.1686(0.0636) & 0.0815(0.0259) & 0.0429(0.0119) \\ 
\hline
\hline
\multirow{4}{*}{T} & $\multirow{2}{*}{$\bC$}$ & RMRE & 0.0438(0.0563) & 0.0117(0.0015) & 0.0086(0.0009) & 0.0070(0.0006)  \\ 
& & SRRE & 0.2068(0.0091) & 0.1388(0.0106) & 0.0762(0.0084) & 0.0328(0.0026) \\ 
\cline{3-7}
& $\multirow{2}{*}{$\gamma$}$ & RMRE & 0.1686(0.2022) & 0.0549(0.0172) & 0.0395(0.0129) & 0.0334(0.0107) \\
& & SRRE & 0.7381(0.2099) & 0.3768(0.1274) & 0.1844(0.0625) & 0.0716(0.0219) \\
\hline
\hline
\multirow{4}{*}{Coss} & $\multirow{2}{*}{$\bC$}$ & RMRE & 0.0368(0.0497) & 0.0113(0.0013) & 0.0083(0.0009) & 0.0065(0.0006) \\ 
& & SRRE & 0.1952(0.0083) & 0.1305(0.0086) & 0.0724(0.0075) & 0.0316(0.0023) \\  
\cline{3-7}
& $\multirow{2}{*}{$\gamma$}$ & RMRE & 0.1386(0.1771) & 0.0547(0.0199) & 0.0398(0.0141) & 0.0321(0.0093) \\ 
& & SRRE & 0.7226(0.2423) & 0.3494(0.1134) & 0.1672(0.0564) & 0.0683(0.0223) \\ 
\hline 
\end{tabular}}
\end{table}

\begin{table}[ht]
\caption{(Continued) Estimation errors for 2D shapes in the normal model. The means (standard deviation) of the RMSEs of $\hat \bC$ and $\hat \gamma$ with 100 repetitions of RMRE and SRRE are reported. }
\label{table:2d-2}
\resizebox{\linewidth}{!}{
\begin{tabular}{|c|c|c|c|c|c|c|}
\hline 
Shape & Para. & Method & $n=300$ & $n=500$ & $n=700$ & $n=1000$ \\ 
\hline \hline
\multirow{4}{*}{Triangle} & $\multirow{2}{*}{$\bC$}$ & RMRE & 0.1777(0.0092) & 0.1036(0.0169) & 0.0615(0.0053) & 0.0546(0.0029) \\
& & SRRE & 0.1826(0.0076) & 0.1439(0.0059) & 0.1163(0.0052) & 0.0903(0.0033) \\ 
\cline{3-7}
& \multirow{2}{*}{$\gamma$} & RMRE & 0.5981(0.1759) & 0.2938(0.0993) & 0.1456(0.0456) & 0.1031(0.0297) \\ 
&  & SRRE & 0.6317(0.2171) & 0.4015(0.1238) & 0.264(0.0762)  & 0.1655(0.0587) \\
\hline \hline
\multirow{4}{*}{Circle} & \multirow{2}{*}{$\bC$} & RMRE & 0.2513(0.0101) & 0.1480(0.0195)  & 0.0529(0.0064) & 0.0413(0.0020)  \\
& & SRRE & 0.2244(0.0125) & 0.1571(0.0080)  & 0.1184(0.0058) & 0.0839(0.0039) \\ 
\cline{3-7}
& \multirow{2}{*}{$\gamma$} & RMRE & 0.8917(0.3156) & 0.4064(0.1590)  & 0.1303(0.0464) & 0.0880(0.0259) \\
& & SRRE & 0.8034(0.2861) & 0.4431(0.1342) & 0.2768(0.1018) & 0.1659(0.0565) \\ 
\hline \hline
\multirow{4}{*}{Butterfly} & \multirow{2}{*}{$\bC$} & RMRE & 0.3033(0.0036) & 0.2726(0.0067) & 0.2368(0.0106) & 0.1561(0.0124) \\ 
& & SRRE & 0.2894(0.0062) & 0.2602(0.0065) & 0.2312(0.0057) & 0.1954(0.0050)  \\ 
\cline{3-7}
& \multirow{2}{*}{$\gamma$} & RMRE & 1.0743(0.3808) & 0.7234(0.2492) & 0.5368(0.1548) & 0.3276(0.1036) \\ 
& & SRRE & 1.0141(0.3773) & 0.7087(0.2265) & 0.5362(0.1688) & 0.3928(0.1322) \\ \hline
\end{tabular}}
\end{table}

In the second experiment, we generate a class of synthetic signals under various ranks and sparsity levels, and apply RMRE and SRRE to these signals to compare their performance. Specifically, the matrix predictors $\bX_i\in\reals^{64\times 64}$ are matrices with each entry is standard normal and the vector predictors $\bZ_i\in\reals^5\sim\mathcal{N}(0,\mathbf{I})$ are random 5-dimensional vectors. The number of sample is chosen as $n=500$. Moreover, we let $\gamma^*=(1,1,1,1,1)^T$ and $\bC^*=\bC_1\bC_2^T$, where $\bC_d\in\{0,1\}^{64\times r},d=1,2$. Each entry of $\bC_d$ is taken as 1 with probability $\sqrt{1-(1-s)^{1/r}}$ and 0 otherwise. The parameter $r$ controls the rank of $\bC$ and $s$ controls its sparsity. In the experiment, we choose $r=1,5,10,20$ and $s=0.01,0.05,0.1,0.2,0.5$. We consider a  normal model where the response $y_i$ is generated as $\langle \bX_i,\bC^*\rangle+\gamma^{*T}\bZ_i+\epsilon_i$, where $\epsilon_i$ is standard normal.

We evaluate the performance of RMRE and SRRE with respect to prediction. We trained both methods on a train set with $500$ samples and evaluated the prediction errors on an independent test set with sample size $500$. The prediction error is defined as the misclassification rate of the responses in the logistic model and is defined as RMSE of the responses in the normal model. The regularization parameter $\lambda$ is chosen as the prediction error on the test set is minimized and this approach is applied for both methods for fairness. We reported the mean and the standard deviations of the prediction errors with 100 repetitions in Table~\ref{table:matrix-syn-3} for the normal model. The results for RMSEs of the matrix coefficient $\bC^*$ are reported in Table~\ref{table:matrix-syn-4}.
Compared with SRRE, our proposed method has good estimation results for the rank-1 matrices in both normal and logistic models. As the sparsity decreases, our method outperforms SRRE for matrix of various ranks. When the sparsity parameter $s$ is less or equal than $0.2$, our method at least has a similar performance both in RMSEs in estimating $\bC^*$ and prediction in the responses. The only exception case is for large sparsity matrices ($s=0.5$ and $r\geq 5$), SRRE has smaller estimation errors.


\begin{table}[ht]
\caption{Estimation result for synthetic data in the logistic model. Reported are the mean and standard deviation (in the parenthesis) of the rate of prediction error of $\hat{Y}_i$ out of 100 data replications. RMRE is the proposed method, regularized matrix-variate regression estimator, and SRRE is the spectral regularized regression estimator. }
\label{table:matrix-syn-1}
\begin{center}
\resizebox{\linewidth}{!}{
\begin{tabular}{|cc|c|c|c|c|}
\hline
Sparsity & Method & \multicolumn{4}{c|}{Rank} \\ \cline{3-6}
& & $r=1$ & $r=5$ & $r=10$ & $r=20$  \\ \hline \hline
$1\%$    & RMRE   & 0.1178(0.0218) & 0.2504(0.0449) & 0.2999(0.0359) & 0.3346(0.0346) \\
         & SRRE   & 0.2351(0.0298) & 0.3454(0.0302) & 0.3641(0.0252) & 0.3755(0.0260)  \\ \hline \hline
$5\%$    & RMRE   & 0.0994(0.0178) & 0.3460(0.0338) & 0.3899(0.0255) & 0.3994(0.0247) \\
         & SRRE   & 0.2034(0.0226) & 0.3688(0.0287) & 0.3952(0.0230) & 0.4053(0.0256) \\ \hline \hline
$10\%$   & RMRE   & 0.1192(0.0234) & 0.3744(0.0317) & 0.4059(0.0233) & 0.4098(0.0230) \\
         & SRRE   & 0.2012(0.0260) & 0.3704(0.0258) & 0.3986(0.0227) & 0.4092(0.0226) \\ \hline \hline
$20\%$   & RMRE   & 0.1347(0.0216) & 0.3799(0.0259) & 0.4060(0.0279) & 0.4146(0.0204) \\
         & SRRE   & 0.1978(0.0257) & 0.3474(0.0303) & 0.3828(0.0297) & 0.4008(0.0249)\\ \hline \hline
$50\%$   & RMRE   & 0.1293(0.0192) & 0.3605(0.0262) & 0.3935(0.0226) & 0.4107(0.0206) \\
         & SRRE   & 0.1915(0.0253) & 0.2844(0.0316) & 0.3145(0.0351) & 0.3467(0.0306)  \\ \hline
\end{tabular}}
\end{center}
\end{table}

\begin{table}[ht]
\caption{Estimation result for synthetic data in the logistic model. Reported are the mean and standard deviation (in the parenthesis) of the RMSEs of $\hat{\bC}$ out of 100 data replications. RMRE is the proposed method, regularized matrix-variate regression estimator, and SRRE is the spectral regularized regression estimator. }
\label{table:matrix-syn-2}
\begin{center}
\resizebox{\linewidth}{!}{
\begin{tabular}{|cc|c|c|c|c|}
\hline
Sparsity & Method & \multicolumn{4}{c|}{Rank} \\ \cline{3-6}
 & & $r=1$ & $r=5$ & $r=10$ & $r=20$ \\ \hline \hline
$1\%$    & RMRE   & 0.0697(0.0250) & 0.0842(0.0199) & 0.0921(0.0188) & 0.0954(0.0140) \\
         & SRRE   & 0.0860(0.0256) & 0.0915(0.0189) & 0.0970(0.0183) & 0.0984(0.0137) \\ \hline \hline
$5\%$    & RMRE   & 0.1920(0.0328) & 0.2225(0.0286) & 0.2238(0.0264) & 0.2235(0.0223) \\
         & SRRE   & 0.2090(0.0317) & 0.2251(0.0280) & 0.2248(0.0264) & 0.2238(0.0222)\\ \hline \hline
$10\%$   & RMRE   & 0.2817(0.0431) & 0.3301(0.0367) & 0.3315(0.0273) & 0.3345(0.0301) \\
         & SRRE   & 0.3019(0.0464) & 0.3327(0.0376) & 0.3321(0.0272) & 0.3350(0.0301) \\ \hline \hline
$20\%$   & RMRE   & 0.4090(0.0447) & 0.4972(0.0411) & 0.5159(0.0398) & 0.5105(0.0411) \\
         & SRRE   & 0.4421(0.0466) & 0.5021(0.0420) & 0.5176(0.0400) & 0.5110(0.0411)\\ \hline \hline
$50\%$   & RMRE   & 0.6568(0.0384) & 0.9765(0.0672) & 1.0326(0.0727) & 1.0458(0.0635)\\
         & SRRE   & 0.6932(0.0388) & 0.9855(0.0678) & 1.0343(0.0725) & 1.0464(0.0634)\\ \hline
\end{tabular}}
\end{center}
\end{table}

\begin{table}[ht]
\caption{Estimation result for synthetic data in the normal model. Reported are the mean and standard deviation (in the parenthesis) of the RMSEs of responses $\hat{Y}_i$ out of 100 data replications. RMRE is the proposed method,  regularized matrix-variate regression estimator, and SRRE is the spectral regularized regression estimator. }
\label{table:matrix-syn-3}
\begin{center}
\resizebox{\linewidth}{!}{
\begin{tabular}{|cc|c|c|c|c|}
\hline
Sparsity & Method & \multicolumn{4}{c|}{Rank}\\ \cline{3-6}
 & & $r=1$ & $r=5$ & $r=10$ & $r=20$ \\ \hline \hline
$1\%$    & RMRE   & 1.0664(0.0433) & 1.5320(0.5631) & 2.5858(1.0295) & 5.0129(1.1507) \\
         & SRRE   & 2.1023(0.2856) & 5.1372(1.1028) & 5.6477(1.0383) & 5.9977(0.9132) \\ \hline \hline
$5\%$    & RMRE   & 1.0925(0.0483) & 4.4520(3.5651) & 11.7671(1.9385) & 13.2728(1.5353)\\
         & SRRE   & 3.1843(0.4954) & 12.0552(1.4539) & 12.8586(1.4045) & 13.5944(1.3929)\\ \hline \hline
$10\%$   & RMRE   & 1.1874(0.0907) & 11.8509(5.5520) & 19.1758(2.0377) & 20.0153(1.6428)\\
         & SRRE   & 4.2707(0.9060) & 17.4163(1.8125) & 19.1120(1.8045) & 19.7686(1.5268)\\ \hline \hline
$20\%$   & RMRE   & 1.5175(0.1847) & 25.2633(3.1545) & 29.7988(2.5752) & 31.1058(2.6314)\\
         & SRRE   & 5.7842(1.2161) & 25.2321(2.1572) & 27.9915(2.2142) & 29.8019(2.4316)\\ \hline \hline
$50\%$   & RMRE   & 2.9755(0.3597) & 52.1691(4.8101) & 59.9110(4.4410) & 63.3346(4.5564)\\
         & SRRE   & 8.3832(1.6463) & 41.1190(2.9915) & 49.2962(3.3167) & 54.4517(3.3581)\\ \hline
\end{tabular}}
\end{center}
\end{table}

\begin{table}[ht]
\caption{Estimation result for synthetic data in the normal model. Reported are the mean and standard deviation (in the parenthesis) of the RMSEs of $\hat{\bC}$ out of 100 data replications. RMRE is the proposed method, regularized matrix variate regression estimator, and SRRE is the spectral regularized regression estimator. }
\label{table:matrix-syn-4}
\begin{center}
\resizebox{\linewidth}{!}{
\begin{tabular}{|cc|c|c|c|c|}
\hline
Sparsity & Method & \multicolumn{4}{c|}{Rank} \\ \cline{3-6}
&  & $r=1$ & $r=5$ & $r=10$ & $r=20$ \\ \hline \hline
$1\%$    & RMRE   & 0.0053(0.0009) & 0.0174(0.0102) & 0.0368(0.0170) & 0.0765(0.0179)\\
         & SRRE   & 0.0285(0.0047) & 0.0784(0.0174) & 0.0860(0.0160) & 0.0923(0.0135)\\ \hline \hline
$5\%$    & RMRE   & 0.0067(0.0010) & 0.0661(0.0569) & 0.1815(0.0303) & 0.2049(0.0222)\\
         & SRRE   & 0.0471(0.0079) & 0.1878(0.0225) & 0.1989(0.0215) & 0.2101(0.0196)\\ \hline \hline
$10\%$   & RMRE   & 0.0099(0.0024) & 0.1836(0.0876) & 0.2975(0.0316) & 0.3114(0.0265)\\
         & SRRE   & 0.0648(0.0139) & 0.2708(0.0275) & 0.2969(0.0262) & 0.3080(0.0244)\\ \hline \hline
$20\%$   & RMRE   & 0.0177(0.0035) & 0.3955(0.0475) & 0.4640(0.0370) & 0.4833(0.0369)\\
         & SRRE   & 0.0888(0.0189) & 0.3932(0.0294) & 0.4354(0.0317) & 0.4627(0.0331)\\ \hline \hline
$50\%$   & RMRE   & 0.0433(0.0055) & 0.8153(0.0672) & 0.9296(0.0593) & 0.9862(0.0630)\\
         & SRRE   & 0.1284(0.0244) & 0.6426(0.0440) & 0.7645(0.0436) & 0.8468(0.0459)\\ \hline
\end{tabular}}
\end{center}
\end{table}

\subsection{Simulation data: matrix-variate logistic regression}

We consider a  logistic regression model where the response $y_i$ is $1$ with probability $\frac{1}{1+e^{-\langle \bX_i,\bC^*\rangle-\gamma^{*T}\bZ_i}}$ and 0 otherwise.
We evaluate the performance of our method (RMRE) and SRRE, and report the mean and the standard deviations of the predication errors under 100 replications in Table~\ref{table:matrix-syn-1} for the logistic model. The results for RMSEs of the matrix coefficient $\bC$ are reported in Table~\ref{table:matrix-syn-2}.

\subsection{Simulation data analysis: robust matrix-variate regression}

We compare the proposed robust matrix variate regression with the robust low-rank matrix estimator (RLRME)~\citep{elsener2018robust}, which can be considered a robust version of SRRE by replacing the quadratic loss with the Huber loss.
The value of the tuning parameter is suggested as $\lambda=2\sqrt{\log(m+q)/q/n}$ in \citet{elsener2018robust}. For a fair comparison, we instead use the same method to tune the parameter in the following simulations. We trained both of Robust RMRE and RLRME on a train set with certain number of samples and evaluated the prediction errors on an independent validation set with the sample number of samples.  The results are in Tables 1 to 5. The responses in both of the train set and the validation set are generated by 
$y=\langle \bX,\bC^*\rangle+\gamma^{*T}\bZ+(1-p)\epsilon+p\epsilon'$, where $\epsilon'\sim\mathcal{N}(0,10^4)$ is the contamination with probability $p$. The results are in Tables 6 and 7. The value of $\gamma^*$ is set to $(1,1,1,1,1)^T$ for all simulations in this section.

We use RMSEs of $\bC$ and $\gamma$ with the various sample sizes to compare the proposed method, robust regularized matrix-variate regression estimator (Robust RMRE), with RLRME, which requires a certain number of samples to make an accurate estimation. Yet, our proposed method does not have such a requirement and has smaller errors even the sample size is small. In this experiment, the true signal $\bC^*\in\reals^{30\times 30}$ is a square shape and is defined as $\bC^*_{ij}=1$ if $10\leq i,j\leq 20$ and $\bC^*_{ij}=0$ otherwise. 
The noise is set as $\epsilon\sim\mathcal{N}(0,1)$ and the contamination probability is $p=0.05$ and $p=0.1$.
Fig.~\ref{fig:msre-sample-size} shows the RMSEs of the estimated $\bC$ and $\gamma$ against the sample size with the corrupted responses. The response has the corruption with probability 0.05 in the left figure of Fig.~\ref{fig:msre-sample-size} and this probability is taken as 0.1 in the right figure. In both figures, RLRME has large errors when the sample size is smaller than 1000 and our method (RobustRMRE) always has small errors even the sample size is also small. This phenomenon can be explained by Theorem~\ref{thm:asymptotics}, where it is shown that Robust RMRE has a better consistency rate than RLRME when $r$ is smaller than $m$ and $q$.

We further compare the robust estimators RobustRMRE and RLRME with non-robust estimator RMRE for various 2D shapes (Triangle, Circle, Butterfly). The same 2D shape signals in \citet{zhou2013tensor} and \citet{zhou2014regularized} are used here but the matrix size is changed to $40\times 40$. The sample size is taken as 1000 and the responses are corrupted with probability $0.05$. Fig.~\ref{fig:_triangle_circle_butterfly-robust} exhibits that RobustRMRE can recover better signals than RLRME while the non-robust estimator has a poor estimation when the outliers are present. The numerical results are reported in Table~\ref{table:robust_gaussian_noise} with the various sample sizes. The responses of the training and validation sets are corrupted with probability $0.05$. We conducted 100 repetitions of the simulation and reported the mean and the standard deviation of the RMSEs. Overall, the errors of our method, RobustRMRE, are much smaller than RLRME and the non-robust estimator, RMRE. RLRME is robust to the outliers for most of the shapes except for some simple shapes such as the square. Both RobustRMRE and RLRME have smaller errors as the sample size is increasing. An additional comparison is made when the noise distribution is chosen as the standard Cauchy distribution. The RMSEs of $\mathbf{C}$ and $\gamma$ against the sample size with Cauchy noise are shown in Fig.~\ref{fig:msre-sample-size-cauchy}. The experimental results reported in Table~\ref{table:robust_cauchy} also indicate that our method RobustRMRE outperforms than RLRME when the noise distribution is heavy-tailed.

\begin{figure}[ht]
   \centering
   \includegraphics[width=0.45\linewidth]{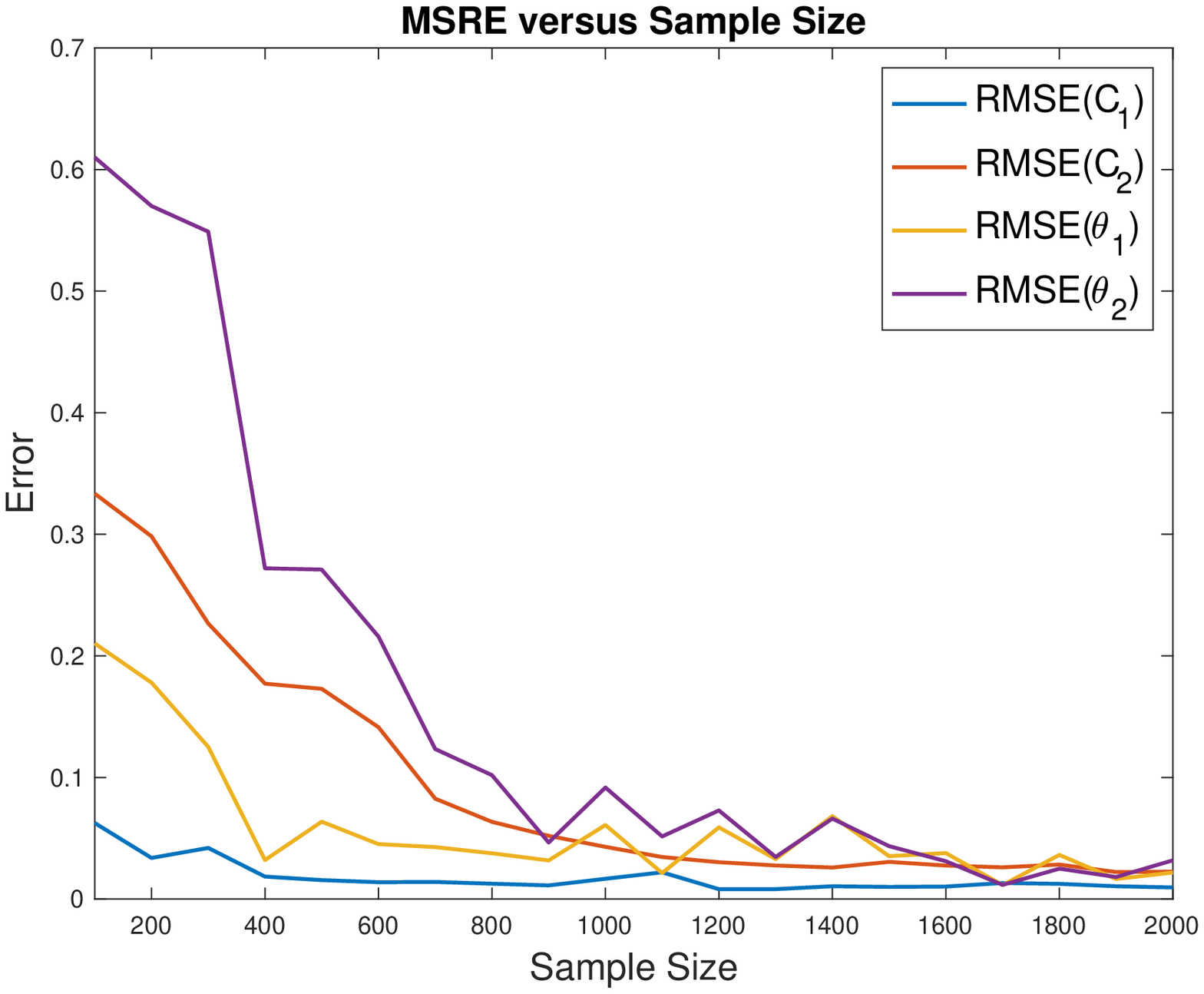}
   \includegraphics[width=0.45\linewidth]{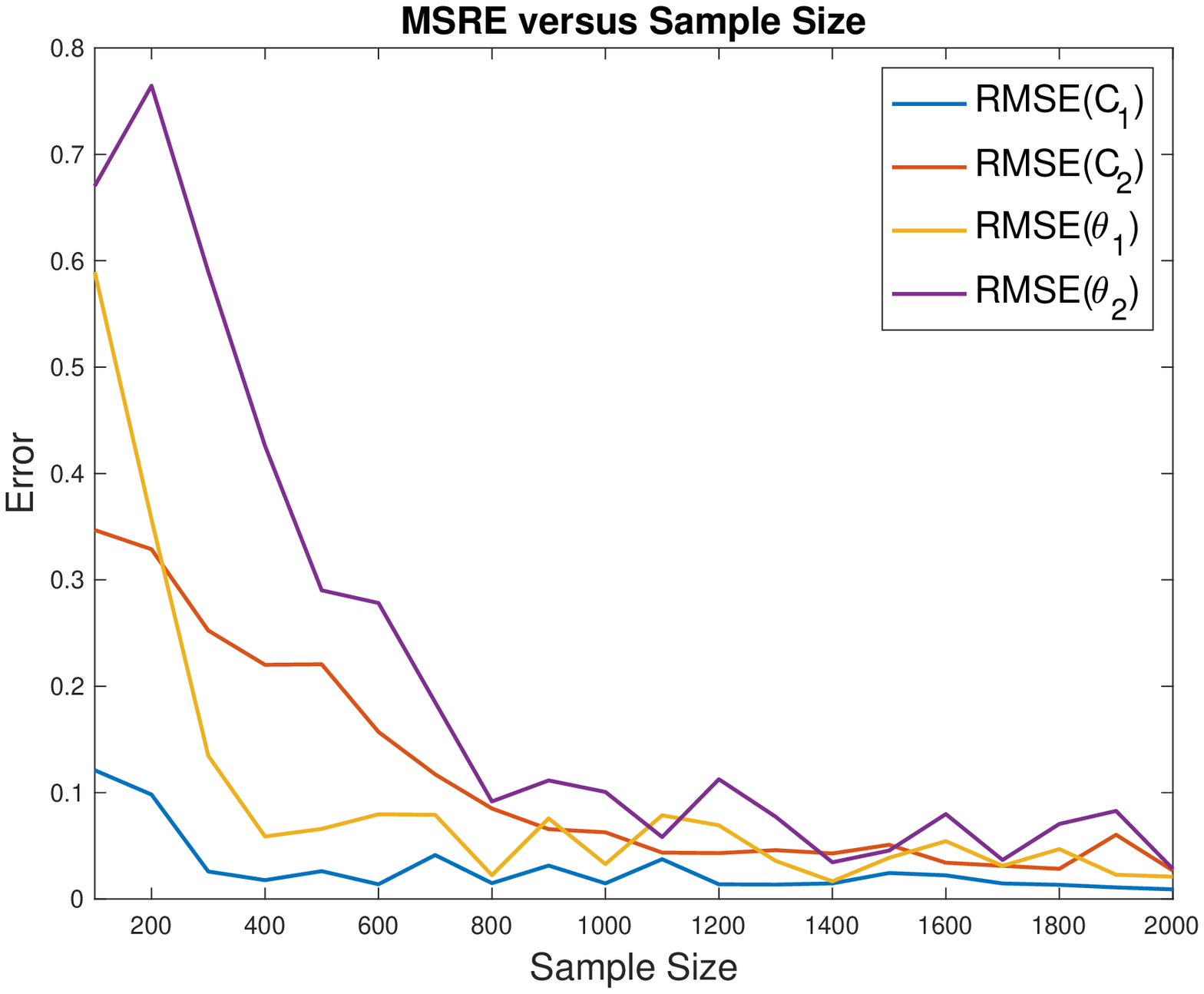}
   \caption{Comparison of the proposed robust estimator (Robust RMRE) with RLRME as the number of samples increases. The outliers are generated with probability 0.05 (left) and 0.1 (right). Here $C_1:=\hat{\bC}^{RobustRMRE}, C_2:=\hat{\bC}^{RLRME}, \theta_1:=\hat{\gamma}^{RobustRMRE},\theta_2:=\hat{\gamma}^{RLRME}$ denote the estimated values of $\bC$ and $\gamma$. The curves are obtained out of 10 replications. }
   \label{fig:msre-sample-size}
\end{figure}

\begin{figure}[ht]
   \centering
   \includegraphics[width=0.6\linewidth]{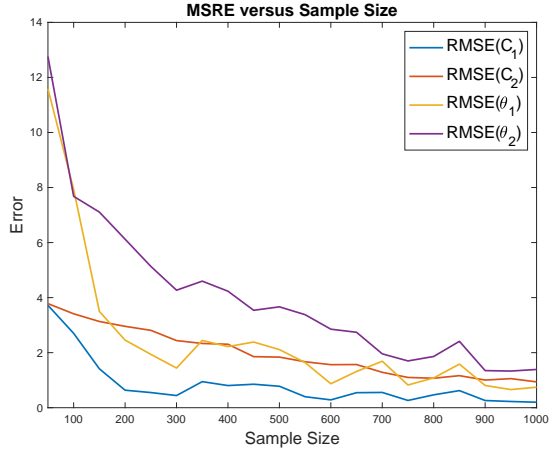}
   \caption{Comparison of the proposed robust estimator (Robust RMRE) with RLRME as the number of samples increases. The noise follows the standard Cauchy distribution. Here $C_1:=\hat{\bC}^{RobustRMRE}, C_2:=\hat{\bC}^{RLRME}, \theta_1:=\hat{\gamma}^{RobustRMRE},\theta_2:=\hat{\gamma}^{RLRME}$ denote the estimated values of $\bC$ and $\gamma$. The curves are obtained out of 10 replications. }
   \label{fig:msre-sample-size-cauchy}
\end{figure}

\begin{figure}[ht]
  \centering
  \includegraphics[width=\linewidth]{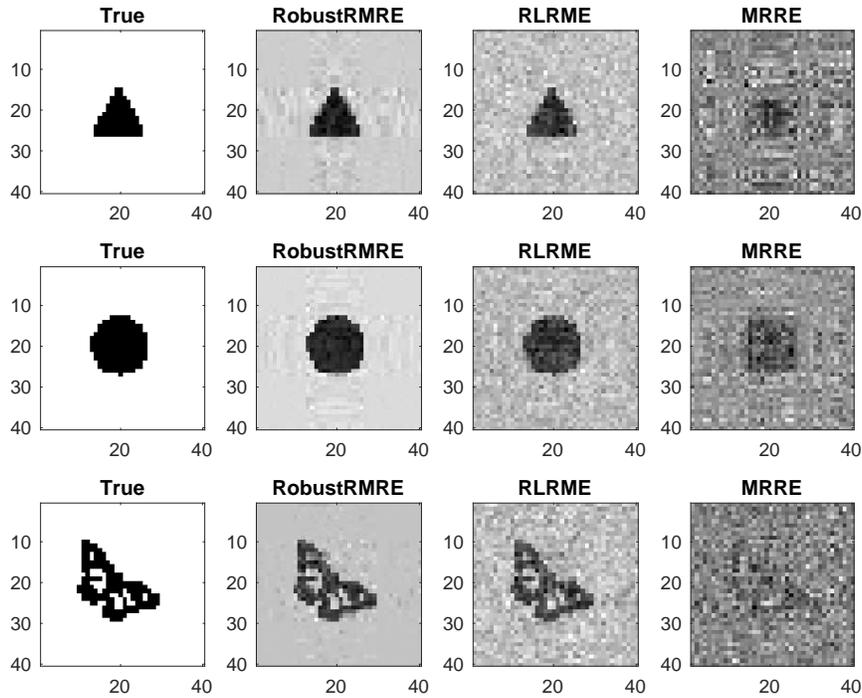}
  \caption{Comparison of the robust estimators RobustRMRE and RLRME with non-robust estimator. The sample size is 1000. }
  \label{fig:_triangle_circle_butterfly-robust}
\end{figure}

\begin{table}[ht]
\caption{Estimation errors for 2D shapes in the normal model with corruption probability 0.1. The means (standard deviation) of the RMSEs of $\hat \bC$ and $\hat \gamma$ with 100 repetitions are reported. }
\label{table:robust_gaussian_noise}
\resizebox{\linewidth}{!}{
\begin{tabular}{|c|c|c|c|c|c|c|}
\hline
Shape & Para. & Method & $n=300$ & $n=500$ & $n=700$ & $n=1000$\\
\hline \hline
\multirow{6}{*}{Triangle} & \multirow{3}{*}{$\bC$} & RobustRMRE & 0.1440(0.0150) & 0.0885(0.0168) & 0.0600(0.0077) & 0.0622(0.0062) \\ 
& & RLRME & 0.2172(0.0116) & 0.1990(0.0092)  & 0.1727(0.0064) & 0.1427(0.0057) \\ 
& & RMRE & 0.6123(0.1349) & 0.8563(0.1645) & 0.7337(0.1248) & 0.5503(0.0650) \\
\cline{2-7} 
& \multirow{3}{*}{$\gamma$} & RobustRMRE & 0.3419(0.1251) & 0.1851(0.0711) & 0.1106(0.0419) & 0.0899(0.0347) \\ 
& &  RLRME  & 0.6011(0.1842) & 0.4174(0.1220)  & 0.3609(0.1090)  & 0.2900(0.0810) \\ 
& & RMRE & 1.8787(0.7706) & 1.7311(0.7102) & 1.3096(0.4453) & 0.9755(0.3208) \\ 
\hline \hline
\multirow{6}{*}{Circle} &\multirow{3}{*}{$\bC$}    & RobustRMRE & 0.1924(0.0159) & 0.1075(0.0185) & 0.0606(0.0098) & 0.0403(0.0032) \\ 
& & RLRME & 0.2763(0.0082) & 0.2490(0.0099)  & 0.2166(0.0083) & 0.1796(0.0084) \\ 
& & RMRE & 0.6470(0.1362)  & 0.9113(0.1528) & 0.7231(0.1106) & 0.5448(0.0644) \\ 
\cline{2-7} 
& \multirow{3}{*}{$\gamma$} & RobustRMRE & 0.4352(0.1286) & 0.2192(0.0761) & 0.1079(0.0416) & 0.0670(0.0227)  \\ 
& & RLRME      & 0.6440(0.1726)  & 0.4870(0.1541)  & 0.3972(0.1222) & 0.3192(0.0920)  \\ 
& & RMRE       & 1.9324(0.7693) & 1.8282(0.7066) & 1.2984(0.4442) & 0.9762(0.3591)  \\ 
\hline \hline
\multirow{6}{*}{Butterfly} &\multirow{3}{*}{$\bC$} & RobustRMRE & 0.2766(0.0095) & 0.2309(0.0103) & 0.1632(0.0209) & 0.1149(0.0137) \\ 
&  & RLRME & 0.2967(0.0094) & 0.2726(0.0112) & 0.2410(0.0081)  & 0.1996(0.0076) \\ 
& & RMRE & 0.4828(0.0621) & 0.6332(0.0928) & 0.8299(0.1162) & 0.9296(0.1105) \\ 
\cline{2-7} 
& \multirow{3}{*}{$\gamma$} & RobustRMRE & 0.5529(0.1547) & 0.4366(0.1175) & 0.2683(0.0966) & 0.1761(0.0574) \\ 
& & RLRME & 0.6431(0.1860)  & 0.5522(0.1642) & 0.4051(0.1167) & 0.3338(0.1027) \\ 
& & RMRE & 1.5754(0.5818) & 1.4130(0.5340)   & 1.3980(0.5066)  & 1.3328(0.4371) \\ \hline 
\end{tabular}}
\end{table}

\begin{table}[ht]
\caption{Estimation errors for 2D shapes in the normal model as the noise follows standard Cauchy distribution. The means (standard deviation) of the RMSEs of $\hat \bC$ and $\hat \gamma$ with 100 repetitions are reported. }
\label{table:robust_cauchy}
\resizebox{\linewidth}{!}{
\begin{tabular}{|c|c|c|c|c|c|c|}
\hline 
Shape & Para. & Method & $n=300$ & $n=500$ & $n=700$ & $n=1000$ \\ 
\hline \hline
\multirow{6}{*}{Triangle} & \multirow{3}{*}{C} & RobustRMRE & 0.1643(0.0128) & 0.1171(0.0130)  & 0.0685(0.0143)  & 0.0607(0.0085) \\ 
& & RMRE & 1.7292(4.3080) & 3.1972(6.7697) & 4.6344(24.463) & 3.0650(9.2047)\\ 
& & RLRME & 0.2219(0.0052)  & 0.2011(0.0054) & 0.1837(0.0057) & 0.1580(0.0048)\\ 
\cline{2-7}
& \multirow{3}{*}{$\gamma$} & RobustRMRE & 0.3992(0.1189) & 0.2530(0.0990) & 0.1486(0.0623) & 0.1182(0.0456) \\ 
& & RMRE & 5.5389(12.511) & 6.4430(12.667) & 7.9979(42.893) & 5.0391(15.485) \\ 
& & RLRME & 0.5862(0.1694) & 0.4667(0.1285) & 0.4025(0.1099) & 0.3076(0.1011) \\ 
\hline \hline
\multirow{6}{*}{Circle} & \multirow{3}{*}{C} & RobustRMRE & 0.2158(0.0162) & 0.1985(0.0154) & 0.0923(0.0142) & 0.0611(0.0257) \\
& & RMRE & 1.4626(5.3714) & 5.0158(14.751) & 2.0511(2.8052) & 12.075(98.787) \\ 
& & RLRME & 0.2847(0.0052) & 0.2593(0.0072) & 0.2320(0.0078) & 0.1987(0.0118) \\ 
\cline{2-7}
& \multirow{3}{*}{$\gamma$} & RobustRMRE & 0.5134(0.1821) & 0.3874(0.1319) & 0.1917(0.0660)  & 0.1263(0.0908) \\
& & RMRE & 3.7961(9.1388) & 9.0904(21.233) & 3.9441(5.6502) & 15.248(115.30) \\ 
& & RLRME & 0.6609(0.1794) & 0.5335(0.1282) & 0.4126(0.1229) & 0.3333(0.1068) \\ 
\hline \hline
\multirow{6}{*}{Butterfly} & \multirow{3}{*}{C} & RobustRMRE & 0.2834(0.0081) & 0.2950(0.0255) & 0.2035(0.0132) & 0.1264(0.0128) \\
& & RMRE & 1.1195(2.8638) & 2.2435(5.7978) & 9.5673(44.680) & 3.1074(4.5313) \\ 
& & RLRME & 0.2946(0.0047) & 0.2706(0.0058) & 0.2461(0.0070) & 0.2143(0.0061) \\ 
\cline{2-7}
& \multirow{3}{*}{$\gamma$} & RobustRMRE & 0.5635(0.1931) & 0.5175(0.1635) & 0.3371(0.1005) & 0.1942(0.0755) \\
& & RMRE & 4.3562(13.464) & 4.7767(10.976) & 13.589(49.213) & 4.7824(10.478)\\
& & RLRME & 0.6330(0.1829) & 0.4982(0.1299) & 0.4534(0.1238) & 0.3440(0.1035)\\
\hline
\end{tabular}}
\end{table}

\subsection{Real Data Analysis}

We compared RMRE, SRRE, and SDNCMV on the electroencephalography (EEG) data of alcoholism \citep{zhang1995event}. 
There are 77 individuals with alcoholism and 44 individuals as the control (non-alcoholic). The subjects are exposed to a stimulus and the voltage values were measured from 64 channels of electrodes placed on the subject's scalp for 256 time points and 120 trials. We average these sampling data out of 120 trials and collect 122 matrices of size $256\times 64$ of the individuals. The responses of these individuals are taken as either 0 (non-alcoholic) or 1 (alcoholic). 
The classical linear model deals with vector-valued predictors, which will lead to poor performance if we simply vectorize the matrix sampling data into a vector. For instance, the dimension of the sampling unit is $256\times64=16,384$ but the sample size is only 122. Moreover, vectorization destroys the wealth of structural information inherently possessed in the matrix data~\citep{zhou2014regularized}.

We apply RMRE, SRRE, and SDNCMV to the data and evaluate the prediction error on the testing set via the cross-validation. We used the same modified cross-validation method in \cite{zhou2014regularized} to achieve a fairer way in tuning the regularization parameter. Specifically, we divided the full data into a training and a testing
sample using $k$-fold cross-validation. 
Then for the training data, we further employed a 5-fold cross-validation
to tune the shrinkage parameter $\lambda$. We then applied the tuned model that is fully based
on the training data now to the testing data and evaluated the misclassification error rate
for testing. We report the average performance on the random partition of the training and testing sets in cross validation with $100$ repetitions.

\begin{table}[t!]
\caption{Misclassification rates (standard deviation in parentheses)  of RMRE, SRRE, and SDNCMV for the EEG data}
\label{table:eeg-data}
\resizebox{\linewidth}{!}{\begin{tabular}{|c|c|c|c|c|}
\hline
Method & Leave-one-out CV & 5-fold CV & 10-fold  CV & 20-fold CV \\ \hline \hline
\textbf{RMRE} & 0.2131 & 0.2298 (0.0255) & 0.2255 (0.0214) & 0.2186 (0.0104) \\
SRRE & 0.2131 & 0.2332 (0.0225) & 0.2228 (0.0188) & 0.2171 (0.0122) \\
SDNCMV &  0.2459 & 0.2348 (0.0064) & 0.2349 (0.0071) & 0.2610 (0.0086) \\
\hline
\end{tabular}}
\end{table}

We also apply RMRE, SRRE, and SDNCMV to classify the IEEE leucorrhoea microscopic images \citep{hao2019automatic},
which consist of 6 categories: Erythrocytes (Ery, also known as red blood cells), Leukocytes (Leu, also known as white blood cells), Molds, Epithelial Cells (Epi), Pyocytes (Pyo) and Trichomonads. Examples of these images are present in Figure~\ref{fig:leu-sample} where Trichomonads are not found in this dataset. To evaluate the  performance, we sample $120$ images from the categories of Leu and Pyo ($60$ images each) and downsample the resolution of each image to $32$ by $32$.  
The comparisons on the EEG data and the leucorrhoea data for leave-one-out, 5, 10, and 20-fold cross-validation are reported in Tables~\ref{table:eeg-data} and \ref{table:leucorrhoea-data}, where the mean misclassification rates and its standard deviation over 100 runs are recorded (the standard deviation of leave-one-out CV is zero as the training/testing partition and the algorithms are deterministic). 

\begin{figure}[ht]
  \centering
  \includegraphics[width=0.19\linewidth,height=0.12\textheight]{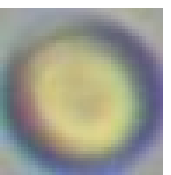}
  \includegraphics[width=0.19\linewidth,height=0.12\textheight]{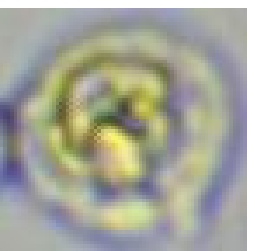}
  \includegraphics[width=0.19\linewidth,height=0.12\textheight]{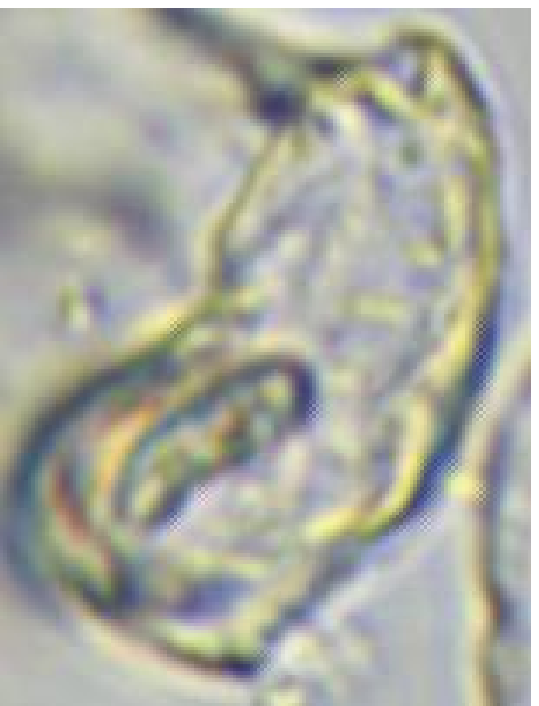} 
  \includegraphics[width=0.19\linewidth,height=0.12\textheight]{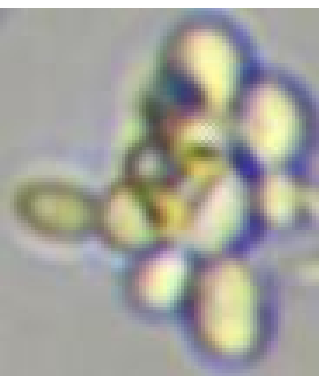} 
  \includegraphics[width=0.19\linewidth,height=0.12\textheight]{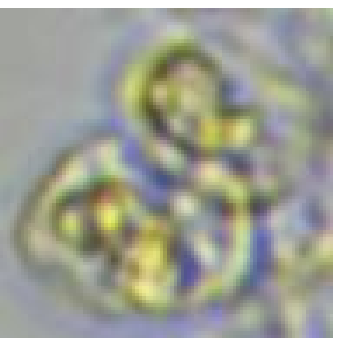}
  \caption{Microscopic images for each category. The types of images are shown as Erythrocytes (Ery, \nth{1}), Leukocytes (Leu, \nth{2}), Epithelial Cells (Epi, \nth{3}), Mildew (Mid, \nth{4}), Pyocytes (Pyo, \nth{5})}
  \label{fig:leu-sample}
\end{figure}


\begin{table}[t!]
\caption{Misclassification rates (standard deviation in parentheses)  of RMRE, SRRE, and SDNCMV for the leucorrhoea data}
\label{table:leucorrhoea-data}
\resizebox{\linewidth}{!}{
\begin{tabular}{|c|c|c|c|c|}
\hline
Method & Leave-one-out CV & 5-fold CV & 10-fold  CV & 20-fold CV \\ \hline \hline
\textbf{RMRE} & 0.2750 & 0.2293 (0.0213) & 0.2537 (0.0204) & 0.2528 (0.0172)\\
SRRE & 0.2583 & 0.2605 (0.0251) & 0.2593 (0.0149) & 0.2582 (0.0161) \\
SDNCMV & 0.2583 & 0.2583 (0.0081) & 0.2461 (0.0106) & 0.2653 (0.0076) \\
\hline
\end{tabular}}
\end{table}

\section{Summary}
The proposed method, RMRE, has a better or comparable performance in general.  Compared with SDNCMV, RMRE performs better in the EEG data and is comparable in the leucorrhoea data. Compared with SRRE, RMRE performs better in the leucorrhoea data and is comparable in the EEG data.  In addition, empirically we observe that both RMRE and SRRE run much faster than SDNCMV. This is expected as SDNCMV minimizes an energy similar to \eqref{eq:optimization} without the rank constraint, which means there are $qm+p$ parameters to be estimated. In comparison, RMRE and SRRE only estimate $r(q+m)+p$ parameters, which is much smaller when $r$ is small.
Both SRRE and RLRME are estimators with nuclear norm-based regularization that promote low-rankness. Therefore, our numerical experiments show the advantage of our method of ``double regularization'' that is based on both low-rank constraint and sparse regularization. In addition, the performance of our estimator is not sensitive to the choice of the rank $r$ and the estimation errors are usually stable for a large range of choices of the rank $r$.  

\section*{Acknowledgement}
The authors thank the editor, referees, and Yiyuan She for their valuable comments.
This work was partially supported by NSF grants (DMS-1924792 and CNS-1818500).

\bibliographystyle{unsrtnat}
\bibliography{matrix_reg_ref}

\appendix

\section{Appendix}

\textbf{Proof of Proposition~\ref{prop:convergence}}
\begin{proof}
  By the line search rule, we have that $F(\bC^{(k+1)},\gamma^{(k+1)})\leq F(\bC^{(\iter)},\gamma^{(\iter)})$ for all $k\geq 1$. As a result, the limit $\lim_{k\rightarrow\infty}F(\bC^{(\iter)},\gamma^{(\iter)})$ exists. Assume that one of the limiting point of the sequence $(\bC^{(\iter)},\gamma^{(\iter)})$ is $(\tilde{\bC},\tilde{\gamma})$, then the line search rule implies that $\frac{\partial}{\partial_{\gamma}} F(\tilde{\bC},\tilde{\gamma})=0$ and $\frac{\partial}{\partial_{\bC}} F(\tilde{\bC},\tilde{\gamma})=0$. 
  \end{proof}
\textbf{Proof of Lemma~\ref{lemma:general_manifold}}
\begin{proof}
We assume that for any $\bx$ in the neighborhood of $\bx^*$, $\bx-\bx^*$ can be uniquely decomposed into $\bx-\bx^*=\bx^{(1)}+\bx^{(2)}$ such that $\bx^{(1)}\in T_{\bx^*}(\calM)$ and $\bx^{(2)}\in T_{\bx^*,\perp}(\calM)$. Let $b=\|\bx-\bx^*\|$, then if $b\leq c_0$, then $\|\bx^{(1)}\|\leq b$ and $\|\bx^{(2)}\|\leq C_Tb^2$.
Let $\bv=\frac{\bx-\bx^*}{\|\bx-\bx^*\|}$ be the direction from $\bx^*$ to $\bx$, then $f(\bx)-f(\bx^*)=\int_{\bx^*}^{\bx}\langle\bv,\nabla f(\bt)\rangle\di\bt=\langle\bx-\bx^*, \nabla f(\bx^*) \rangle+\int_{\bx^*}^{\bx}\langle\bv,\nabla f(\bt)-\nabla f(\bx^*)\rangle\di\bt$, where the first term can be bounded by
\begin{align*}
\langle\bx-\bx^*, \nabla f(\bx^*) \rangle=\langle\bx^{(1)}+\bx^{(2)}, \nabla f(\bx^*) \rangle
&=\langle\bx^{(1)},\Pi_{T_{\bx^*}(\calM)}\nabla f(\bx^*) \rangle+\langle\bx^{(2)},\Pi_{T_{\bx^*}(\calM),\perp}\nabla f(\bx^*) \rangle\\
&\leq b \|\Pi_{T_{\bx^*}(\calM)}\nabla f(\bx^*)\|+C_Tb^2\|\Pi_{T_{\bx^*}(\calM),\perp}\nabla f(\bx^*)\|.
\end{align*}
On the other hand, the second term can be bounded by
$
\int_{\bx^*}^{\bx}\langle\bv,\nabla f(\bt)-\nabla f(\bx^*)\rangle\di\bt\geq \frac{1}{2}b^2C_{H,1}.
$
Combining these two inequalities, the lemma is proved.

\end{proof}

\textbf{Proof of Lemma~\ref{lemma:curvature}}
\begin{proof}
  By following the  tangent space  of the set of low-rank matrices in the literature \citep{Absil2015,10.5555/3291125.3309642}, we have the following explicit expressions of the tangent space of $\calM$ at $(\bC^*,\gamma^*)$ that
  $
  T_{(\bC^*,\gamma^*),\calM}=\{(\bA\bV^*\bV^{*T}+\bU^*\bU^{*T}\bB,\by): \bA,\bB\in\reals^{q\times m}, \by\in\reals^p \},
  $
  where $\bU^*\in\reals^{q\times r}$ and $\bV^*\in\reals^{m\times r}$ are obtained from the singular value decomposition of $\bC^*$ such that $\bC^*=\bU^*\Sigma\bV^{*T}$. The projection operators in Lemma~\ref{lemma:general_manifold} are given by
  \begin{align*}
  \Pi_{T_{(\bC^*,\gamma^*),\calM}}(\bD,\by)&=(\bU^*\bU^{*T}\bD+\bD\bV^*\bV^{*T}-\bU^*\bU^{*T}\bD\bV^*\bV^{*T},\by),\\  \Pi_{T_{(\bC^*,\gamma^*),\calM},\perp}(\bD,\by)&=((\bI-\bU^*\bU^{*T})\bD(\bI-\bV^*\bV^{*T}),0).
  \end{align*}

By choosing $\bU^{*\perp}\in\reals^{q\times (q-r)}$  such that $[\bU^{*\perp},\bU^*]\in\reals^{q\times q}$ is orthogonal and choose $\bV^{*\perp}\in\reals^{m\times (m-r)}$  such that $[\bV^{*\perp},\bV^*]\in\reals^{m\times m}$ is orthogonal, then we can express the projectors as follows: for any $(\bC,\gamma)$ close to $(\bC^*,\gamma^*)$, we may write $\bC-\bC^*=\bU^*\bD_1\bV^{*T}+\bU^{*\perp}\bD_2\bV^{*T}+\bU^{*}\bD_3\bV^{^*\perp T}+\bU^{*\perp}\bD_4\bV^{^*\perp T}$,
$
\|\Pi_{T_{(\bC^*,\gamma^*),\calM},\perp}(\bC-\bC^*,\gamma-\gamma^*)\|=\|(\bU^{*\perp T}\bD_4\bV^{^*\perp},0)\|=\|\bD_4\|_F,
$
and
$
\|\Pi_{T_{(\bC^*,\gamma^*),\calM}}(\bC-\bC^*,\gamma-\gamma^*)\|=\sqrt{\|\bD_1\|_F^2+\|\bD_2\|_F^2+\|\bD_3\|_F^2+\|\gamma-\gamma^*\|^2}.
$
By $\rank(\bD)=r$, we have $\bD_4=\bD_2(\bD_1+\bU^{*T}\bC^*\bV{^*})^{-1}\bD_3$. Thus,
when $\|\bC-\bC^*\|_F\leq \sigma_r(\bC^*)/2$,
\[
\|\Pi_{T_{(\bC^*,\gamma^*),\calM},\perp}(\bC-\bC^*,\gamma-\gamma^*)\|\leq \frac{\|\bD_2\|_F\|\bD_3\|_F}{\sigma_r(\bD_1)-\|\bC-\bC^*\|_F}\leq \frac{2\|\Pi_{T_{(\bC^*,\gamma^*),\calM}}(\bC-\bC^*,\gamma-\gamma^*)\|^2}{\sigma_r(\bC^*)},
\]
and Lemma~\ref{lemma:curvature} is proved.
\end{proof}

\begin{proof}[Proof of Theorem~\ref{thm:asymptotics}]
In the proof,  we mainly work with
$
f(\bC,\gamma)= \sum_{i=1}^n l(y_i,\langle \Xb_i,\Cb\rangle+\gamma^T\zb_i),
$
and it is sufficient to show that for all $(\bC,\gamma)\in\calM$ such that 
$
 \sqrt{\|\bC-\bC^*\|_F^2+\|\gamma-\gamma^*\|^2}\geq C_{error,1}$ and $ f(\bC,\gamma)-f(\bC^*,\gamma^*)\geq \lambda P(\bC^*,\gamma^*).  $

To prove it, we first calculate the constants and the operators in Lemma~\ref{lemma:general_manifold} as follows. For all three models, the constant on the curvature of $\calM$ is the same. Hence,  we may choose $C_T=2/\sigma_{\min}(\bC^*)$.
 In addition, as discussed in the proof of Lemma~\ref{lemma:curvature}, the projectors $\Pi_{T}$ and $\Pi_{T,\perp}$ at $(\bC^*,\gamma^*)$ can be defined by
 \begin{align}\label{eq:projection}
  \Pi_{T_{(\bC^*,\gamma^*)}(\calM)}(\bC,\gamma)=(\bC-\Pi_{\bU^*,\perp}\bC\Pi_{\bV^*,\perp},\gamma)\mbox{ and }
  \Pi_{T_{(\bC^*,\gamma^*)}(\calM),\perp}(\bC,\gamma)=(\Pi_{\bU^*,\perp}\bC\Pi_{\bV^*,\perp},0),
 \end{align}
 where $\bU^*\in\reals^{q\times r}$ and $\bV^*\in\reals^{m\times r}$ are the left and right singular components of $\bC^*$, $\Pi_{\bU^*}=\bU^*\bU^{*T}$, $\Pi_{\bU^*,\perp}=\bI-\Pi_{\bU^*}$, $\Pi_{\bV^*}=\bV^*\bV^{*T}$, and $\Pi_{\bU^*,\perp}=\bI-\bV^*\bV^{*T}$. As for the first derivative, we have
$
\nabla f(\bC^*,\gamma^*)=\sum_{i=1}^nw_{1,i}\mathrm{vec}(\bX_i,\bz_i),
$
where
\begin{align*}
w_{1,i}=\begin{cases}2\epsilon_i,\,\,\text{for the matrix variate regression model};\\2\frac{\epsilon_i}{\max(|\epsilon_i|.\alpha)},\,\,\text{for the robust matrix variate regression model};\\\epsilon_i,\,\,\text{for the logistic matrix variate regression model}.\end{cases}
\end{align*}
Combining it with \eqref{eq:projection},
\begin{align*}
  \Pi_{T_{(\bC^*,\gamma^*)}(\calM)}\nabla f(\bC^*,\gamma^*)&=(\Pi_{\bU}\bX_i+\bX_i\Pi_{\bV}-\Pi_{\bU}\bX_i\Pi_{\bV},\bz_i),\\
  \Pi_{T_{(\bC^*,\gamma^*)}(\calM),\perp}\nabla f(\bC^*,\gamma^*)&=(\bX_i-\Pi_{\bU}\bX_i-\bX_i\Pi_{\bV}+\Pi_{\bU}\bX_i\Pi_{\bV},0).
\end{align*}
\begin{lemma}\label{lemma:projection_randomvector}
For any projection matrix $\bU\in\reals^{n\times d}$ and a random vector $\bx\in\reals^n$ with each element i.i.d. sampled from a sub-Gaussian distribution of parameter $\sigma_0$, then for $t\geq 2$,
\[ 
\Pr(\|\bx^T\bU\|\geq t\sigma_0\sqrt{d})\leq C\exp(-Ct).
\]
\end{lemma}
\begin{proof}
This lemma follows from the McDiarmid's inequality~\citep[Theorem 3]{maurer2021concentration}. In particular, we have that
$
\Expect \|\bx^T\bU\|\leq \sqrt{\Expect [\bx^T\bU\bU^T\bx]}=\sqrt{\Expect \left[\sum_{i=1}^n\bx_i^2\sum_{j=1}^d\bU_{ij}^2\right]}\leq \sigma_0^2d,
$
and let $\bx^{(i)}\in\reals^n$ be defined such that $\bx^{(i)}_j=\bx_j$ if $j\neq i$ and $\bx^{(i)}_i=0$, then $|\|\bx^T\bU\|-\|\bx^{(i)T}\bU\||\leq |\bx_i|\|\bU(i,:)\|$, where $\|\bU(i,:)\|$ represents the norm of the $i$-th row of $\bU$. As a result, $\|\bx^T\bU\|-\|\bx^{(i)T}\bU\|$ is sub-Gaussian with parameter $\sigma_0\|\bU(i,:)\|$. Combining it with the fact that  $\sum_{i=1}^n\|\bU(i,:)\|^2=d$ and the sub-Gaussian version of the McDiarmid's inequality~\citep[Theorem 3]{maurer2021concentration}, the lemma is proved.
\end{proof}

The Assumption 1 and Lemma~\ref{lemma:projection_randomvector} imply that with a probability of at least $1-C\exp(-Cn)-C\exp(-Ct(r(q+m)+p))$, 
\begin{align}\label{eq:concentration}
\Big\|\sum_{i=1}^nw_{1,i} \mathrm{vec}\Big(\Pi_{T_{(\bC^*,\gamma^*)}(\calM)}\nabla f(\bC^*,\gamma^*)\Big)\Big\|&\leq tC_1t\sigma_0\sqrt{n(r(q+m)+p)} ,\\
\Big\|\sum_{i=1}^nw_{1,i} \mathrm{vec}\Big(\Pi_{T_{(\bC^*,\gamma^*)}(\calM),\perp}\nabla f(\bC^*,\gamma^*)\Big)\Big\|&\leq tC_1t\sigma_0\sqrt{n(qm-r(q+m)+p)} \nonumber,
\end{align}
where
\begin{align*}
\sigma_0=\begin{cases}\sigma,\,\,\text{for the matrix variate regression model;}\\
1,\,\,\text{for the robust matrix variate regression model;}\\1,\,\,\text{for the logistic matrix variate regression model}.\end{cases}
\end{align*}

Therefore, for the Hessian matrix defined in \eqref{eq:hessian}, we have $C_{H,1}=C_2n$. Plug in Lemma~\ref{lemma:general_manifold}, we have that for $b=\sqrt{\|\hat{\bC}-\bC^*\|_F^2+\|\hat{\gamma}-\gamma^*\|^2}$, we have
\begin{align*}
&f(\bC,\gamma)-f(\bC^*,\gamma^*)\geq \frac{1}{2}b^2C_{2} - b \|\Pi_{T_{\bx^*}(\calM)}\nabla f(\bx^*)\|-C_Tb^2\|\Pi_{T_{\bx^*}(\calM),\perp}\nabla f(\bx^*)\|
\end{align*}
which is larger than $\lambda P$ if
$
\frac{b^2C_{2}}{6}\geq \max\left(\lambda P, b \|\Pi_{T_{\bx^*}(\calM)}\nabla f(\bx^*)\|, C_Tb^2\|\Pi_{T_{\bx^*}(\calM),\perp}\nabla f(\bx^*)\|\right),
$
or equivalently, if $C_2\sqrt{n}\geq 6 C_1t\sigma_0\sqrt{(qm+p)} $, and
\begin{align}\label{eq:thmasymtptotics1}
b\geq \max\left(\frac{6C_1t\sigma_0\sqrt{n(r(q+m)+p)} }{C_2},\sqrt{\frac{6\lambda P(\bC^*,\gamma^*)}{C_2}}\right).
\end{align}
As a result,
we have 
\begin{align}
    f(\bC,\gamma)-f(\bC^*,\gamma^*)\geq \lambda P(\bC^*,\gamma^*) 
    \mbox{ for all } \{(\bC,\gamma)\in\calM: \dist((\bC,\gamma),(\bC^*,\gamma^*))\in \mathcal{I}\},\label{eq:estimation1}
\end{align}
where\[
  \mathcal{I}=\Bigg[\max\left(\frac{6C_1t\sigma_0\sqrt{n(r(q+m)+p)} }{C_2},\sqrt{\frac{6\lambda P(\bC^*,\gamma^*)}{C_2}}\right),c_0\Bigg].
\]

Next, for all $(\bC,\gamma)$ such that $\dist((\bC,\gamma),(\bC^*,\gamma^*))=\sqrt{\|{\bC}-\bC^*\|_F^2+\|{\gamma}-\gamma^*\|^2}=b$ where \begin{equation}\label{eq:conditionb}b\leq c_0, b\geq \frac{4C_1t\sigma_0\sqrt{n(qm+p)} }{C_2 n},  \frac{1}{4}C_2n b^2\geq \lambda P(\bC^*,\gamma^*)\end{equation} we have
\begin{align*}
  f(\bC,\gamma)-f(\bC^*,\gamma^*)\geq \frac{1}{2}C_2n b^2- b \|\nabla f(\bC^*,\gamma^*)\|
  \geq \frac{1}{2}C_2n b^2-bC_1t\sigma_0\sqrt{n(qm+p)} \geq \lambda P(\bC^*,\gamma^*) .
\end{align*}
By \eqref{eq:assumption_asymptotics}, there exists $b$ satisfying \eqref{eq:conditionb} and we have $f(\bC,\gamma)-f(\bC^*,\gamma^*)\geq \lambda P(\bC^*,\gamma^*)$ for all $\{(\bC,\gamma):\dist((\bC,\gamma),(\bC^*,\gamma^*))=b\}$. Since $f$ is convex, $f(\bC,\gamma)-f(\bC^*,\gamma^*)\geq \lambda P(\bC^*,\gamma^*)$ holds for all $\{(\bC,\gamma):\dist((\bC,\gamma),(\bC^*,\gamma^*))\geq b\}$. 
Combining it with \eqref{eq:estimation1}, we have that for all $(\bC,\gamma)\in\calM$ such that $\sqrt{\|\bC-\bC^*\|_F^2+\|\gamma-\gamma^*\|^2}\geq C_{error,1}$, $f(\bC,\gamma)-f(\bC^*,\gamma^*)\geq \lambda P(\bC^*,\gamma^*)$, and the theorem is proved.

\end{proof}

\textbf{Proof of Theorem~\ref{prop:convergence2}}
\begin{proof}
First, 
by Assumption 2, for all $\left\{(\bC,\gamma):\sqrt{\|\bC-\bC^*\|_F^2+\|\gamma-\gamma^*\|^2}=c_0'^2\right\}$, $F(\bC,\gamma)\geq f(\bC,\gamma) > \frac{C_2n}{2}c_0'^2 + \sum_{i=1}^n l(y_i,\langle \Xb_i,\Cb^*\rangle+\gamma^{*T}\zb_i$. Since $F(\bC^{(\iter)}, \gamma^{(\iter)})$ is nonincreasing, we can choose a small initial step size $\alpha_0>0$, such that if the initial step size $\alpha$ in line search satisfies $\alpha<\alpha_0$, then $(\bC^{(\iter)}, \gamma^{(\iter)})\in\mathcal{B}$ for all $\iter\geq 1$, where $\mathcal{B}=\{(\bC,\gamma): \dist((\bC,\gamma),(\bC^*,\gamma^*))\leq b\}$.

 By the proof of \eqref{eq:concentration} we have
\begin{align*}
\|T_{(\bC^*,\gamma^*),\calM}\nabla f(\bC^*,\gamma^*)\|&\leq nC_1t\sigma_0\sqrt{n(r(q+m)+p)} \\
\|T_{(\bC^*,\gamma^*),\calM,\perp}\nabla f(\bC^*,\gamma^*)\|&\leq
nC_1t\sigma_0\sqrt{n(qm-r(q+m)+p)} .
\end{align*}
Then, for $(\bC,\gamma)$ with $\dist((\bC,\gamma),(\bC^*,\gamma^*))=b$, Assumption 2 implies that if we let $\bv\in\reals^{qm+p}$ be the unit vector $\frac{\mathrm{vec}((\bC^*,\gamma^*)-(\bC,\gamma))}{\|\mathrm{vec}((\bC^*,\gamma^*)-(\bC,\gamma))\|}$, then  
\begin{align*}
\langle\bv, T_{(\bC^*,\gamma^*),\calM}\nabla f(\bC,\gamma)\rangle&\geq nC_2b-C_1t\sigma_0\sqrt{n(r(q+m)+p)} \\
\langle\bv, T_{(\bC^*,\gamma^*),\calM,\perp}\nabla f(\bC,\gamma)\rangle&\leq
nC_3b+C_1t\sigma_0\sqrt{n(qm-r(q+m)+p)} .
\end{align*}
That is, if 
\begin{align}\label{eq:theorem2_conditions0}
  nC_2b-C_1t\sigma_0\sqrt{n(r(q+m)+p)} >
  C_Tb\left(nC_3b+C_1t\sigma_0\sqrt{n(qm-r(q+m)+p)} \right)+\lambda C_{partial},
\end{align}
and $(\bC^+,\gamma^+)$ represents the outcome of the gradient descent algorithm after one iteration, then \begin{equation}\label{eq:bound}\dist((\bC^+,\gamma^+),(\bC^*,\gamma^*))\leq \dist((\bC,\gamma),(\bC^*,\gamma^*)).\end{equation} To simplify the condition \eqref{eq:theorem2_conditions0}, the note that its sufficient condition is
\begin{align}\label{eq:theorem2_conditions1}
    \frac{1}{4}nC_2b>\max\biggl\{C_1t\sigma_0\sqrt{n(r(q+m)+p)} ,C_TC_3nb^2, C_TC_1t\sigma_0\sqrt{n(qm-r(q+m)+p)} b,\lambda C_{partial}\biggr\},
\end{align}
which is in turn guaranteed by  
\begin{align}
\sqrt{n}>\frac{4C_TC_1t\sigma_0\sqrt{(qm-r(q+m)+p)} }{C_2},\,\,
\frac{4C_1t\sigma_0\sqrt{(r(q+m)+p)} }{C_2\sqrt{n}}+\frac{4\lambda C_{partial}}{nC_2}<b<\frac{C_2}{4C_TC_3}.\label{eq:theorem2_conditions2}
\end{align}
By the  assumptions in Theorem~\ref{prop:convergence2}, \eqref{eq:theorem2_conditions2} is satisfied with  $b=\dist((\bC^{(0)},\gamma^{(0)}),(\bC^*,\gamma^*))$, and \eqref{eq:bound} implies that $(\bC^{(\iter)}, \gamma^{(\iter)})\in\mathcal{B}\;\forall
\;\iter\geq 1.$
It remains to prove \eqref{eq:stationary}, which is similar to the proof of \eqref{eq:bound_consistency}.
\end{proof}

\end{document}